\documentclass{article}

\usepackage{microtype}
\usepackage{graphicx}
\usepackage{subfigure}
\usepackage{booktabs} 
\usepackage{CJKutf8}
\usepackage{amsmath}
\usepackage{amsthm}
\usepackage{amssymb}
\usepackage{natbib}
\usepackage{bm}
\usepackage{multirow}
\usepackage{graphicx}
\usepackage{optidef}
\usepackage{url}
\usepackage{cases}

\newcommand{\bx}{\bm{\mathrm{x}}}
\newtheorem{thm}{Theorem}[section]
\newtheorem{proposition}[thm]{Proposition}
\newtheorem{definition}[thm]{Definition}

\newcommand{\Note}[1]{{\color{black}{#1}}} 

\usepackage{hyperref}

\usepackage[accepted]{icml2020}

\icmltitlerunning{Interpolation between Residual and Non-Residual Networks}

\begin{document}

\twocolumn[
\icmltitle{Interpolation between Residual and Non-Residual Networks}



\icmlsetsymbol{equal}{*}

\begin{icmlauthorlist}
\icmlauthor{Zonghan Yang}{dcs}
\icmlauthor{Yang Liu}{dcs}
\icmlauthor{Chenglong Bao}{yau}
\icmlauthor{Zuoqiang Shi}{math}
\end{icmlauthorlist}

\icmlaffiliation{dcs}{Institute for Artificial Intelligence, Beijing National Research Center for Information Science and Technology, Department of Computer Science and Technology, Tsinghua University.}
\icmlaffiliation{yau}{Yau Mathematical Sciences Center, Tsinghua University.}
\icmlaffiliation{math}{Department of Mathematical Sciences, Tsinghua University}
\icmlcorrespondingauthor{Chenglong Bao}{clbao@mail.tsinghua.edu.cn}

\icmlkeywords{Machine Learning, ICML}

\vskip 0.3in
]



\printAffiliationsAndNotice{}  

\begin{abstract}
Although ordinary differential equations (ODEs) provide insights for designing network architectures, its relationship with the non-residual convolutional neural networks (CNNs) is still unclear. In this paper, we present a novel ODE model by adding a damping term. It can be shown that the proposed model can recover both a ResNet and a CNN by adjusting an interpolation coefficient. Therefore, the damped ODE model provides a unified framework for the interpretation of residual and non-residual networks. The Lyapunov analysis reveals better stability of the proposed model, and thus yields robustness improvement of the learned networks. Experiments on a number of image classification benchmarks show that the proposed model substantially improves the accuracy of ResNet and ResNeXt over the perturbed inputs from both stochastic noise and adversarial attack methods. Moreover, the loss landscape analysis demonstrates the improved robustness of our method along the attack direction.

\end{abstract}

\section{Introduction}

Although deep learning has achieved remarkable success in many machine learning tasks, the theory behind it has still remained elusive. In recent years, developing new theories for deep learning has attracted increasing research interests. One important direction is to connect deep neural networks (DNNs) with differential equations \cite{eproposal} which have been largely explored in mathematics. This line of research mainly contains three perspectives: solving high dimensional differential equations with the help of DNNs due to its high expressive power \cite{Han8505}, discovering a differential equation that identifies the rule of the observed data based on the standard block of existing DNNs \cite{neuralode}, and designing new architectures based on the numerical schemes of  differential equations \cite{rev2,lu-lm, zhu-rk,rev1,nonlocal,lu-transformer}. 

While each attempt in the above directions has strengthened the theoretical understanding of deep learning, there still remain many open questions. Among them, one important question is {\em what is the relationship between differential equations and non-residual convolutional neural networks}. Most prior studies have focused on associating residual networks (ResNets) \cite{resnet} with differential equations \cite{lu-lm,neuralode}, not only because ResNets are relatively easy to optimize and achieve better classification accuracy than CNNs, but also because the skipping connections among layers can be easily induced by the discretization of difference operators in differential equations. However, residual neural networks only account for a small fraction of the entire neural network family and have their own limitations. For example, \citet{compareCVmodels} indicate that ResNets are more sensitive to the perturbation of the inputs and the shallow CNNs. As a result, it is important to move a further step to investigate the relationship between differential equations and non-residual convolutional neural networks.

In this paper, we present a new ordinary differential equation (ODE) that interpolates non-residual and residual CNNs. The ODE is controlled by an interpolation parameter $\lambda$ ranging from 0 to $\infty$. It is equivalent to a residual network when $\lambda$ is 0. On the contrary, the ODE amounts to a non-residual network when $\lambda$ approaches to $\infty$. Hence, our work provides a unified framework for understanding both non-residual and residual neural networks from the perspective of ODE.  The interpolation is able to improve over both non-residual and residual networks. Compared with non-residual networks, our ODE is much easier to optimize, especially for deep architectures. Compared with residual networks, we use the Lyapunov analysis to show that the interpolation results in improved robustness. To achieve the interpolation, a key difference of our work from existing methods is to discretize integral operators instead of difference operators to obtain neural networks. Experiments on image classification benchmarks show that our approach substantially improves the accuracies of ResNet \cite{resnet} and ResNeXt \cite{resnext} when inputs are perturbed by both stochastic noise and adversarial attack methods. Furthermore, the visualization of the loss landscape of our model validates our Lyaponov analysis. 

\section{Related Work}\label{sec:rel}

Interpreting machine learning from the perspective of dynamic systems was firstly advocated by  \citet{eproposal} and \citet{rev2}. Recently, there have been many exciting works in this direction \cite{lu-lm,neuralode}. We briefly review previous methods closely related to architecture design and model robustness.

\textbf{ODE inspired architecture design} Inspired by the relationship between ODE and neural networks, \citet{lu-lm} use a linear multi-step method to improve the model capacity of ResNet-like networks. \citet{zhu-rk} utilize the Runge-Kutta method to interpret and improve DenseNets and CliqueNets. \citet{rev1} and \citet{rev2} leverage the leap-frog method to design novel reversible neural networks. \citet{nonlocal} propose to model non-local neural networks with non-local differential equations. \citet{lu-transformer} design a novel Tranformer-like architecture with Strang-Marchuk splitting scheme. \citet{neuralode} show that blocks of a neural network can be instantiated by arbitrary ODE solvers, in which parameters can be directly optimized with the adjoint sensitivity method. \citet{aneuralode} improve the expressive power of a neural ODE by mitigating the trajectory intersecting problem. Compared to the above works, our work provides a new ODE that unifies the analysis of residual and non-residual networks which leads to an interpolated architecture. The experiments validate the advantages of the proposed method using this framework.

\textbf{ODE and model robustness} A number of previous methods have also been proposed to improve adversarial robustness from the perspective of ODE. \citet{smallResNet} propose to use a smaller step factor in the Euler method for ResNet. \citet{ImplicitResNet} utilize an implicit discretization scheme for ResNet. \citet{RobustNeuralODE} propose to train a time-invariant neural ODE regularized by steady-state loss. \citet{NeuralSDE} and \citet{EnResNet} introduce stochastic noise to enhance its robustness inspired by stochastic differential equations. The aforementioned works have concentrated on improving numerical discretization schemes or introducing stochasticity for ODE modeling to gain robustness. From the Lyapunov stability perspective, \citet{antisymmetricrnn} propose to use anti-symmetric weight matrices to parametrize an RNN, which enhances its long-term dependency. \citet{lu-adv} also accelerate adversarial training by recasting it as a differential game from an ODE perspective. In this work, we provide the Lyaponov analysis of the proposed ODE model which shows the robustness improvements over ResNets in terms of local stability.

\section{Methodology}\label{sec:method}
\Note{In this section, we first introduce the background of the relationship between ODE and ResNets, and then the proposed ODE model and its stability analysis is present.}

\subsection{Background}\label{background}

\Note{Considering the ordinary differential equation:
\begin{equation}
    \frac{{\rm d} \bx(t)}{{\rm d} t} = f(\bx(t), t),~\bx(0) = \bx_0,
\label{resnet_ode}
\end{equation}}
\noindent where $\bx: [0, T] \rightarrow \mathbb{R}^{d}$ represents the state of the system.
Given the discretization step $\Delta t$ and define $t_n = n\Delta t$, the forward Euler method of Eq. \eqref{resnet_ode} becomes
\begin{equation}
    \bx(t_{n+1}) = \bx(t_n) + \Delta t f(\bx(t_n),t_n).
\end{equation}
Let $\bx_n = \bx(t_n)$, $\Delta t = 1$, it recovers a residual block:
\begin{equation}
    \bx_{n+1} = \bx_n + f_n(\bx_n),
\label{resnet}
\end{equation}
and $f_n$ is the n-th layer operation in ResNets. Thus, the output of network is equivalent to the evolution of the state variable at terminal time $T$, i.e.\ $\bx(T)=\bx_{N}$ is the output of last layer in a ResNet if assuming $N=T/\Delta t$. 

The dynamic formulation of ResNets (see Eq. \eqref{resnet_ode}) was initially established in \cite{eproposal}. It inspired many interesting neural network architectures by using different discretization methods the first order derivative in Eq. \eqref{resnet_ode} such as linear multi-step network \cite{lu-lm} and Runge-Kutta network \cite{zhu-rk}. From Eq. \eqref{resnet_ode}, the skip connection from the current step $\bx_n$ to the next step estimation $\bx_{n+1}$ always exists no matter which kind of discretization is applied. Thus, a feedforward CNN without skip connection can not be directly explained under this framework which inspired current work. In the next section, we introduce a damped ODE which bridges the non-residual CNNs and ResNets.

\subsection{The Proposed ODE Model}

Based on the ODE formulation , we add a damping term to the model \eqref{resnet_ode} and leads to the following model:
\begin{equation}
    \frac{{\rm d}\bx(t)}{{\rm d} t} = -\lambda \bx(t) + \rho(\lambda) f(\bx(t), t),
\label{damped_resnet_ode}    
\end{equation}
starting from $\bx(0)=\bx_0$. The constant $\lambda\in[0,+\infty)$ is the called interpolation coefficient and $\rho:[0,+\infty)\mapsto[0,+\infty)$ is the weight function. The following proposition shows that the model shown in Eq. \eqref{damped_resnet_ode} has a closed form solution.
\begin{proposition}\label{prop:solution}
For any $T>0$, the solution of the ODE \eqref{damped_resnet_ode} is
\begin{equation}\label{ODE:solution}
    \bx(T) = e^{-\lambda T}\left(\bx_0 + \rho(\lambda)\int_0^T e^{\lambda t} f(\bx(t),t){\rm d} t\right).
\end{equation}
\end{proposition}
\begin{proof}
Multiplying both sides by $e^{\lambda t}$, it has
\begin{equation*}
    \frac{{\rm d}(e^{\lambda t}\bx(t))}{{\rm d}t}=e^{\lambda t}\frac{{\rm d}\bx(t)}{{\rm d}t} + \lambda e^{\lambda t}\bx(t) = \rho(\lambda) e^{\lambda t} f(\bx(t),t).
\end{equation*}
Integrating within $[0,T]$ yields
\begin{equation}
    e^{\lambda T} \bx(T) - \bx(0) = \rho(\lambda) \int_0^T e^{\lambda t} f(\bx(t),t) {\rm d} t,
\end{equation}
which induces the equality \eqref{ODE:solution}.
\end{proof}
Following from the proposition \ref{prop:solution} and the notations in section \ref{background}, the iterative formula of $\bx_n$ is 
\begin{equation}\label{iter}
    \bx_{n+1} = e^{-\lambda\Delta t}\bx_n + e^{-\lambda t_{n+1}}\rho(\lambda)\int_{t_n}^{t_{n+1}}e^{\lambda t}f(\bx(t),t){\rm d} t.
\end{equation}
Assuming $f(\bx(t),t) = f(\bx_n,t_n)$ for all $t\in[t_n,t_{n+1})$, the iterative scheme in Eq. \eqref{iter} reduces to
\begin{equation}\label{iter:1}
    \bx_{n+1} = e^{-\lambda\Delta t}\bx_n + \frac{1-e^{-\lambda\Delta t}}{\lambda}\rho(\lambda)f_n(\bx_n),
\end{equation}
where $f_n(\bx_n)=f(\bx(t_n),t_n)$ is the convolutions in $n$-th layer.
Now, we are ready to analyze Eq. \eqref{iter:1} by choosing an appropriate weight function $\rho(\lambda)$. When the weight function $\rho(\lambda)$ satisfies
\begin{equation}\label{ass:rho}
    \rho(\lambda)\rightarrow 1, \lambda\rightarrow 0^+\mbox{ and }\rho(\lambda)\sim\lambda, \lambda\rightarrow +\infty,
\end{equation}
the output of $n$-th layer is
\begin{equation}\label{relations}
    \bx_{n+1}= \begin{cases}
    \bx_n + f_n(\bx_n), & \mbox{ if } \lambda\rightarrow0^+, \\
    \Delta t f_n(\bx_n), & \mbox{ if } \lambda\rightarrow+\infty.
    \end{cases}
\end{equation}
The above equation clearly shows that our model recovers ResNets when the interpolation parameter $\lambda$ approaches $0$ and the non-residual CNNs when it approaches $+\infty$. Therefore, the ODE shown in Eq. \eqref{damped_resnet_ode} bridges the residual and non-residual CNNs and inspires the design of new architectures of neural networks.

\subsection{Interpolated Network Design}

Based on the unified ODE model shown in Eq. \eqref{damped_resnet_ode}, two types of $\rho(\lambda)$ are chosen and the corresponding network architectures are proposed. 
Considering the case when $\lambda$ is small, we choose $\rho(\lambda)=1$ and substitute the damping factor $e^{-\lambda \Delta t}$ by its first order approximation:
\begin{equation}\label{app}
    e^{-\lambda\Delta t} \approx 1-\lambda\Delta t.
\end{equation}
Then, from Eq. \eqref{iter:1}, the output of $n$-th layer is
\begin{equation}\label{Net1}
    \bx_{n+1} = (1-\lambda\Delta t)\bx_n + \Delta t f_n(\bx_n).
\end{equation}
To guarantee the positiveness of $\lambda$, we add the ReLU function to the interpolation parameter $\lambda$ and absorb the $\Delta t$ into it. Thus the $n$-th layer of the network is  
\begin{equation}
    {\bx}_{n+1} = (1 - {\rm ReLU}(\lambda_n)){\bx}_n + f_n({\bx}_n).
\label{In-ResNet1}
\end{equation}

\noindent Each $\lambda_n$ is a trainable parameter for the $n$-th layer. It is known that the forward Euler discretization is stable when $\lambda \Delta t \in (0,2)$, i.e. $\lambda \in (0,2/\Delta t)$. As $\Delta t$ in a continuous-time dynamic system is small, the stable range of $\lambda$ can be viewed as a relaxation of $(0, +\infty)$, which coincides with the boundary condition in Eq. \eqref{ass:rho}.

The second choice of the weight function is $\rho(\lambda)=\lambda+1$ which satisfies the assumption in Eq. \eqref{ass:rho}. Using the same approximation in Eq. \ref{app}, the scheme in Eq. \eqref{iter} reduces to 
\begin{equation}\label{Net2}
    \bx_{n+1} = (1-\lambda\Delta t)\bx_n + (1+\lambda\Delta t)f_n(\bx_n).
\end{equation}
Similar as the first choice, the second interpolated network is given by
\begin{equation}
\begin{aligned}
    {\bx}_{n+1} = & (1-{\rm ReLU}(\lambda_n)){\bx}_n  + (1 + {\rm ReLU}(\lambda_n))f_n({\bx}_n).
\end{aligned}  
    \label{In-ResNet2}
\end{equation}

It is easy to know the interpolated networks shown in Eq. \eqref{In-ResNet1} and \eqref{In-ResNet2} recover a non-residual CNN if $\lambda_n=1$ and a Residual network if $\lambda_n= 0$. As  claimed in \cite{resnet, visualization}, the identity shortcut connection helps mitigate gradient vanish problem and makes the loss landscape more smooth. It is natural that when $\lambda \to 0$ in Eq. \eqref{In-ResNet1}, the optimization process of the interpolated model is much better than the non-residual CNN case with the same number of layers.

\subsection{Interpolated Network Improves Robustness}
Despite the high accuracy of ResNets, it is sensitive to the small perturbation of inputs due to the existence of adversarial examples. That is, for a fragile neural network, minor perturbation can accumulate dramatically with respect to layer propagation, resulting in giant shift of prediction. In this section, we show the improvment of the proposed interpolated networks over ResNets. The added damping term in our model weakens the amplitude of the solution of the original ODE. As a result, adding a damping term to the ODE model damps the error propagation process of ResNet, which improves model robustness. 

In the following context, we show that robustness improvement of our proposed networks by using the stability analysis of the ODE. 

\begin{definition} 
Let $\bx^*$ be an equilibrium point of the ODE model \eqref{resnet_ode}. Then $\bx^*$ is called asymptotically locally stable if there exists $\delta>0$ such that $\lim_{t\to+\infty}\|\bx_t-\bx^*\|=0$ for all starting points $\bx_0$ within $\|\bx_0-\bx^*\|\leq \delta$.
\end{definition}
Therefore, the perturbation around equilibrium $\bx^*$ does not change the output the network if $\bx^*$ is asymptotically locally stable.  The next proposition from \cite{lyapunov,chen2001stability} presents a classical method that checks the stability of nonlinear system around the equilibrium when $f$ is time invariant. It is noted that this time invariant assumption may hold as the learned filters in the deep layers converges.
\begin{proposition}

The equilibrium $\bx^{*}$ of the ODE model
\begin{equation}\label{first_order}
    \frac{{\rm d}\bx}{{\rm d}t} = f(\bx(t))
\end{equation}
is asymptotically locally stable if and only if ${\rm Re}(\nu) < 0$ where $\nu$ is the eigenvalue of $\partial_{\bx} f(\bx^*)$ which is the Jacobi matrix of $f$ at $\bx^*$.
\end{proposition}
Considering the damped ODE
\begin{equation}\label{damped_resnet_ode_1}
    \frac{{\rm d}\bx}{{\rm d}t} = -\lambda\bx(t)+\rho(\lambda)f(\bx(t)),
\end{equation}
the Jacobi matrix at the equilibrium $\bx^*$ is
\begin{equation*}
    J_\lambda (\bx^*) = \rho(\lambda)\partial_{\bx} f(\bx^*) -\lambda.
\end{equation*}
Then, the eigenvalues $\hat{\nu}$ of $J_\lambda(\bx^*)$ are
\begin{equation}\label{eigen}
    \rho(\lambda)\nu - \lambda
\end{equation}
where $\nu$ is the eigenvalue of $\partial_{\bx}f(\bx^*)$. When $\rho(\lambda)=1$, we know 
\begin{equation*}
    {\rm Re}(\hat{\nu}) = {\rm Re}(\nu) - \lambda< {\rm Re}(\nu).
\end{equation*}
By choosing positive $\lambda$ properly, we know the ODE in Eq. \eqref{damped_resnet_ode_1} is asymptotically locally stable at $\bx^*$. In general, we know
\begin{equation*}
    {\rm Re}(\hat\nu)<{\rm Re}(\nu) \Leftrightarrow \rho(\lambda) < 1 + \frac{\lambda}{{\rm Re}(\nu)},
\end{equation*}
which coincides with our assumption in Eq. \eqref{ass:rho}. The above analysis shows that the stationary point of our proposed damped ODE model is more likely to be locally stable, and thus improve the its robustness when the input has be perturbed. In the experiments, our loss landscape visualization further validates this analysis.

\begin{table*}[htbp]
    \centering
    \begin{tabular}{c|l||c|c|c|c|c}
    \hline
        Benchmark & Model & Impulse & Speckle & Gaussian & Shot & Avg. \\
    \hline
    \hline
        \multirow{9}{*}{CIFAR-10} & ResNet-110 & 56.38 & 59.12 & 43.82 & 55.47 & 53.70 \\
        ~ & In-ResNet-110 & \textbf{66.32} & \textbf{76.81} & \textbf{71.01} & \textbf{76.55} & \textbf{72.67} \\
        ~ & $\lambda$-In-ResNet-110 & 65.67 & 76.59 & 70.72 & 76.40 & 72.35 \\        
    \cline{2-7}
        ~ & ResNet-164 & 60.88 & 61.77 & 45.66 & 57.75 & 56.51 \\
        ~ & In-ResNet-164 & \textbf{67.95} & 75.96 & 68.95 & 75.31 & \textbf{72.05} \\
        ~ & $\lambda$-In-ResNet-164 & 65.72 & \textbf{76.27} & \textbf{69.74} & \textbf{75.80} & 71.88 \\          
    \cline{2-7}        
        ~ & ResNeXt & 55.12 & 58.21 & 39.14 & 52.06 & 51.13 \\
        ~ & In-ResNeXt & \textbf{55.26} & \textbf{59.87} & \textbf{39.75} & \textbf{54.12} & \textbf{52.25} \\
        ~ & $\lambda$-In-ResNeXt & 51.27 & 57.20 & 37.23 & 51.25 & 49.24 \\        
    \hline
        \multirow{9}{*}{CIFAR-100} & ResNet-110 & 25.36 & 29.69 & 20.16 & 27.81 & 25.76 \\
        ~ & In-ResNet-110 & 32.00 & \textbf{38.81} & 30.00 & 37.71 & 34.63 \\
        ~ & $\lambda$-In-ResNet-110 & \textbf{32.15} & 38.77 & \textbf{30.02} & \textbf{37.82} & \textbf{34.69} \\          
    \cline{2-7}
        ~ & ResNet-164 & 27.55 & 30.90 & 20.40 & 28.97 & 26.95 \\
        ~ & In-ResNet-164 & \textbf{33.05} & \textbf{39.50} & \textbf{29.77} & \textbf{38.17} & \textbf{35.12} \\
        ~ & $\lambda$-In-ResNet-164 & 32.92 & 38.79 & 29.08 & 37.53 & 34.58 \\          
    \cline{2-7}       
        ~ & ResNeXt & \textbf{26.83} & 28.29 & 17.09 & 25.67 & 24.47 \\
        ~ & In-ResNeXt & 25.85 & 29.90 & 18.59 & 27.72 & 25.52 \\
        ~ & $\lambda$-In-ResNeXt & 25.33 & \textbf{31.18} & \textbf{19.88} & \textbf{28.75} & \textbf{26.29} \\        
    \hline
    \end{tabular}
    \caption{Accuracy over the stochastic noise groups from CIFAR-10-C and CIFAR-100-C datasets, corresponded with perturbed CIFAR-10 and CIFAR-100 images from four types of stochastic noise, respectively. All of the results reported are averaged from 5 runs.}
    \label{stochastic-noise}
\end{table*}

\section{Experiments}\label{sec:Exp}

\subsection{Setup}

We evaluate our proposed model on CIFAR-10 and CIFAR-100 benchmarks, training and testing with the originally given dataset. Following \cite{resnet}, we adopt the simple data augmentation technique: padding $4$ pixels on each side of the image and sampling a $32 \times 32$ crop from it or its horizontal flip. For ResNet experiments, we select the pre-activated version of ResNet-110 and ResNet-164 as baseline architectures. For ResNeXt experiments, we select ResNeXt-29, 8 $\times$ 64d as baseline from \cite{resnext}.

We apply Eq. (\ref{In-ResNet1}) to ResNet-110, ResNet-164 and ResNeXt, and refer to them as In-ResNet-110, In-ResNet-164, and In-ResNeXt. We also apply Eq. (\ref{In-ResNet2}) to ResNet-110 and ResNet-164, referring to them as $\lambda$-In-ResNet-110, $\lambda$-In-ResNet-164, and $\lambda$-In-ResNeXt.

The parameters $\lambda_n$ of our interpolation models are initialized by randomly sampling from $\mathcal{U}[0.2, 0.25]$ in ($\lambda$-)In-ResNet-110 and ($\lambda$-)In-ResNeXt, and $\mathcal{U}[0.1, 0.2]$ in In-ResNet-164. The initialization of other parameters in ResNet and ResNeXt follows \cite{resnet} and \cite{resnext}, respectively.

For all of the experiments, we use SGD optimizer with batch size $= \; 128$. For ResNet and ($\lambda$-)In-ResNet experiments, we train for 160 (300) epochs for the CIFAR-10 (-100) benchmark; the learning rate starts with $0.1$, and is divided it by 10 at 80 (150) and 120 (225) epochs. We apply weight decay of 1e-4 and momentum of 0.9. For ResNeXt and ($\lambda$-)In-ResNeXt experiments, the learning rate starts at $0.05$, and is divided it by 10 at 150 and 225 epochs. We apply weight decay of 5e-4 and momentum of 0.9.

We focus on two types of performances: optimization difficulty and model robustness. For optimization difficulty, we test our model on the CIFAR testing dataset. For model robustness, we evaluate the accuracy of our model over the perturbed inputs, details of which are given in the next section. For each experiment, we conduct 5 runs with different random seeds and report the averaged result to reduce the impact of random variations. The standard deviations of reported results can be found in Appendix D.

\subsection{Measuring Robustness}
In this section we introduce the two types of perturbation methods that we use: stochastic noise perturbations and adversarial attacks. For stochastic noise, we leverage the stochastic noise groups in CIFAR-10-C and CIFAR-100-C dataset \cite{CIFAR-C} for testing. The four groups of stochastic noise are impulse noise, speckle noise, Gaussian noise, and shot noise. 
\begin{table}[tbp]
    \centering
    \begin{tabular}{l|c|c}
    \hline
        Model & CIFAR-10 & CIFAR-100 \\
    \hline
    \hline
        ResNet-110 & 93.58 & 72.73 \\
        In-ResNet-110 & 92.28 & 70.55 \\
        $\lambda$-In-ResNet-110 & 92.15 & 70.39 \\
    \hline        
        ResNet-164 & 94.46 & 76.06 \\
        In-ResNet-164 & 92.69 & 72.94 \\
        $\lambda$-In-ResNet-164 & 92.55 & 73.22 \\        
    \hline        
        ResNeXt & 96.35 & 81.63 \\
        In-ResNeXt & 96.48 & 81.64 \\
        $\lambda$-In-ResNeXt & 96.22 & 81.29 \\
    \hline      
    \end{tabular}
    \caption{Accuracy over CIFAR-10 and CIFAR-100 testing data, representing optimization difficulty of each model.  All of the results reported are averaged from 5 runs.}
    \label{acc}
\end{table}
For adversarial attacks, we consider three classical methods: Fast Gradient Sign Method (FGSM), Iterated Fast Gradient Sign Method (IFGSM), and Projected Gradient Descent (PGD). For a given data point $(\bx, y)$:

\begin{itemize}
    \item FGSM induces the adversarial example $\mathbf{x}'$ by moving with step size of $\epsilon$ at each component of the gradient descent direction, namely
    \begin{equation}
        \mathbf{x}' = \bx + \epsilon \cdot {\rm sign}(\nabla_{\bx}\mathcal{L}(\bx, y)).
    \end{equation}
    \item IFGSM performs FGSM with step size of $\alpha$, and clips the perturbed images within $[\bx-\epsilon, \bx+\epsilon]$ iteratively, namely
    \begin{equation}
        \bx^{(m+1)} = {\rm Clip}_{\bx, \epsilon} \left\{ \bx^{(m)} + \alpha \cdot {\rm sign}(\nabla_{\bx}\mathcal{L}(\bx^{(m)}, y)) \right\},
    \end{equation}    
    where $m = 1, 2, \cdots, M$, $\bx^{(0)} = \bx$, and $\bx^{(M)}$ is the induced adversarial image. In our experiments, we set $\alpha = 2/255$ and iteration times $M = 20$.
    \item PGD attack is the same with IFGSM, except that the $\bx^{(0)} = \bx + \delta$ with $\delta \sim \mathcal{U}[-\epsilon, \epsilon]$.
\end{itemize}
\subsection{Results}
\begin{figure}[tbp]
    \centering
    \includegraphics[scale=0.12]{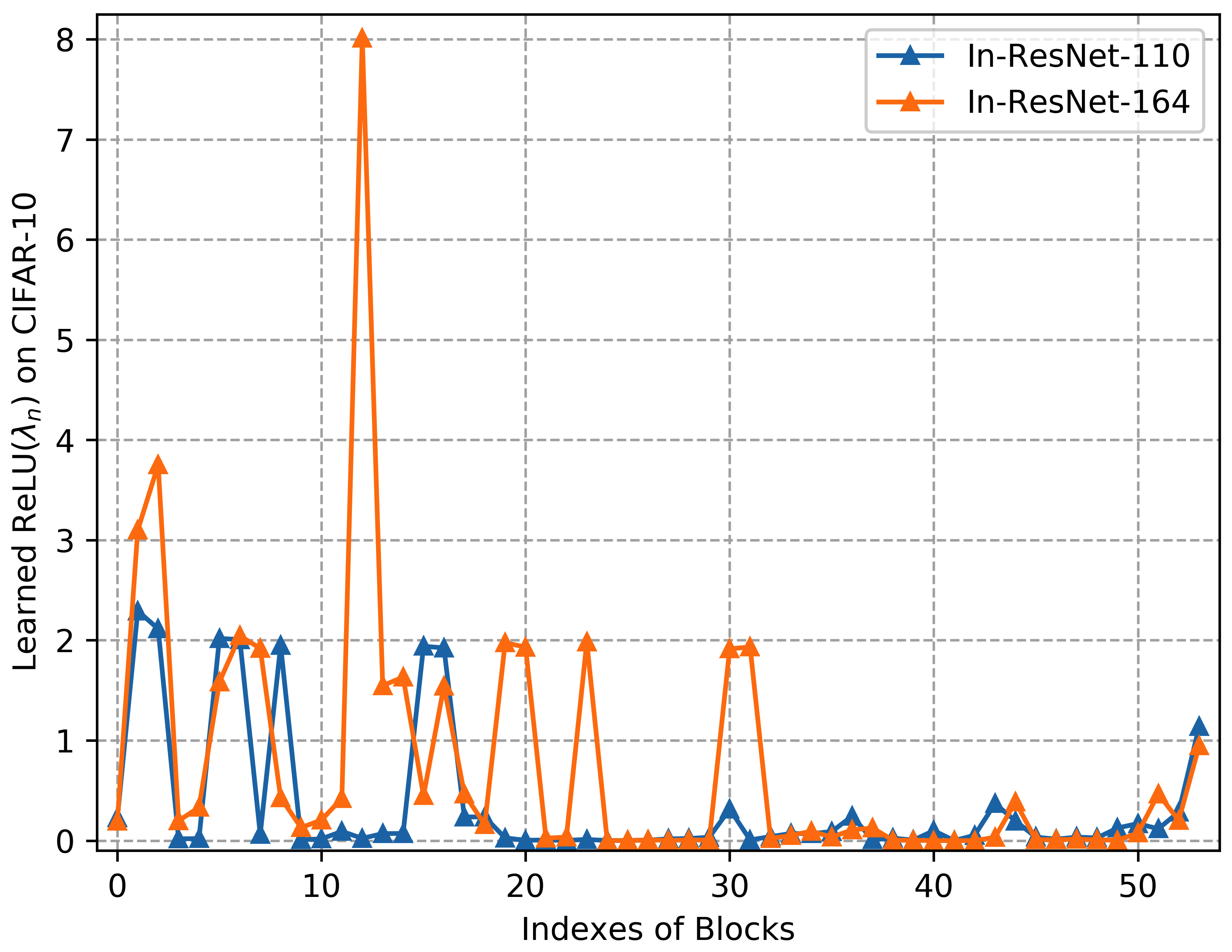}
    \caption{Learned interpolation coefficients in In-ResNet-110 and In-ResNet-164 models trained on CIFAR-10 benchmarks. }
    \label{learned_coeffs}
\end{figure}

\begin{table*}[tbp]
    \centering
    \begin{tabular}{l|l||r|r|r|r|r|r|r|r|r}
    \hline
        \multirow{2}{*}{Benchmark} & \multirow{2}{*}{Model} & \multicolumn{3}{c|}{FGSM} & \multicolumn{3}{|c|}{IFGSM} & \multicolumn{3}{|c}{PGD} \\
    \cline{3-11}
        ~ &  ~ & 1/255 & 2/255 & 4/255 & 1/255 & 2/255 & 4/255 & 1/255 & 2/255 & 4/255 \\
    \hline
    \hline
        \multirow{6}{*}{CIFAR-10} & ResNet-110 & 58.59 & 41.48 & 29.45 & 39.45 & 5.93 & 0.06 & 38.91 & 5.60 & 0.06 \\
        ~ & In-ResNet-110 & \textbf{71.97} & \textbf{55.24} & \textbf{38.26} & 65.70 & \textbf{32.05} & \textbf{5.14} & 65.66 & \textbf{31.74} & \textbf{5.01} \\
        ~ & $\lambda$-In-ResNet-110 & 71.06 & 50.84 & 30.05 & \textbf{65.93} & 30.72 & 3.52 & \textbf{65.81} & 30.45 & 3.41 \\
    \cline{2-11}       
        ~ & ResNet-164 & 63.32 & 44.37 & 30.21 & 46.79 & 8.19 & 0.09 & 46.43 & 7.77 & 0.07 \\
        ~ & In-ResNet-164 & \textbf{70.88} & \textbf{51.84} & \textbf{32.81} & \textbf{64.34} & \textbf{27.43} & \textbf{2.27} & \textbf{64.20} & \textbf{26.95} & \textbf{2.15} \\
        ~ & $\lambda$-In-ResNet-164 & 70.01 & 50.53 & 31.77 & 63.33 & 26.50 & 2.01 & 63.19 & 26.04 & 1.91 \\     
    \hline      
        \multirow{6}{*}{CIFAR-100} & ResNet-110 & 28.01 & \textbf{18.74} & \textbf{14.12} & 15.05 & 2.18 & 0.28 & 14.69 & 2.11 & 0.26 \\
        ~ & In-ResNet-110 & 32.24 & \textbf{18.74} & 11.84 & 23.44 & 4.92 & \textbf{0.55} & 23.22 & 4.81 & \textbf{0.53} \\
        ~ & $\lambda$-In-ResNet-110 & \textbf{32.79} & 18.40 & 11.24 & \textbf{24.17} & \textbf{5.17} & 0.53 & \textbf{24.03} & \textbf{5.00} & 0.51 \\
    \cline{2-11}       
        ~ & ResNet-164 & 35.15 & \textbf{23.58} & \textbf{17.04} & 21.23 & 3.45 & 0.29 & 20.78 & 3.31 & 0.22 \\
        ~ & In-ResNet-164 & 37.21 & 22.30 & 13.93 & 28.05 & 6.59 & \textbf{0.73} & 27.75 & 6.34 & \textbf{0.67} \\
        ~ & $\lambda$-In-ResNet-164 & \textbf{37.37} & 22.50 & 13.94 & \textbf{28.25} & \textbf{6.64} & 0.69 & \textbf{28.03} & \textbf{6.46} & 0.64 \\
    \hline              
    \end{tabular}
    \caption{Accuracy over perturbed CIFAR-10 and CIFAR-100 images from FGSM, IFGSM, and PGD adversarial attacks with different attack radii. All of the results reported are averaged over 5 runs.}
    \label{stochastic-adv}
\end{table*}
\textbf{Optimization difficulty} 
Table \ref{acc} shows the results of In-ResNet-110 and In-ResNet-164 as well as the baselines over CIFAR-10 and CIFAR-100 testing set. On one hand, it can be seen that for ($\lambda$-)In-ResNet-110 and ($\lambda$-)In-ResNet-164, there is accuracy drop within 3 percent compared with the ResNet baselines. This agrees with the fact that the interpolation model may be harder to optimize than ResNet. However, the performance of the interpolation models are still much better than that of the deep non-residual CNN models.

\textbf{Robustness against stochastic noise} Table \ref{stochastic-noise} shows the accuracies of all models over the perturbed CIFAR-10 and CIFAR-100 images from four types of stochastic noise. Our In-ResNet-110 and In-ResNet-164 models achieve substantial improvement over the ResNet-110 and ResNet-164 baselines. For perturbed CIFAR-10 images, accuracy of ($\lambda$-)In-ResNet-110 and ($\lambda$-)In-ResNet-164 are over 15 \% higher than ResNet-110 and ResNet-164 baselines on average. For perturbed CIFAR-100 images, accuracy of ($\lambda$-)In-ResNet-110 and ($\lambda$-)In-ResNet-164 are over 5\% higher than ResNet-110 and ResNet-164 baselines on average. In-ResNeXt models improves the accuracy of the perturbed images over ResNeXt as well.

\textbf{Robustness against adversarial attacks} Table \ref{stochastic-adv} shows the accuracies of all models over the perturbed CIFAR-10 and CIFAR-100 images from FGSM, IFGSM, and PGD attacks at different attack radii of $1/255$, $2/255$, and $4/255$. Most of the robustness results of our ($\lambda$-)In-ResNet-110 and ($\lambda$-)In-ResNet-164 models are higher than those of the ResNet-110 and ResNet-164 models, which is empirically consistent with our Lyapunov analysis. Especially on CIFAR-10 benchmark, our In-ResNet-110 and In-ResNet-164 models obtain significant robustness improvement against the strong IFGSM and PGD attacks at the radii of $1/255$ and $2/255$.

\textbf{Learned interpolation coefficients} To get a better understanding of the interpolation model, we plot the interpolation coefficients $\{{\rm ReLU}(\lambda_n)\}$ in In-ResNet-110 and In-ResNet-164 models trained on CIFAR-10 benchmarks. As shown in Fig \ref{learned_coeffs}, most of the interpolation coefficients lie within the range $[0, \; 1]$, suggesting an interpolating behaviour. According to Eq. (\ref{In-ResNet1}), interpolation coefficients lying within $[1, \; 2]$ represent negative skip connections, with the absolute weight scale of less than $1$. Very few of the interpolation coefficients are larger than $2$, which is in line with the stability range of forward Euler scheme. In general, 79.6\%(72.2\%) of $\lambda_n$'s in In-ResNet-110(164) are larger than 0.01, which accounts for the significance in robustness. More visualizations of learned interpolation coefficients can be found in Appendix A.

\begin{table*}[tbp]
    \centering
    \begin{tabular}{l|r|r|r|r|r}
        \hline
        Model & Acc. & noise & FGSM & IFGSM & PGD \\
        \hline
        \hline
        ResNet-110 & \textbf{93.58} & 53.70 & 41.48 & 5.93 & 5.60 \\
        \hline        
        In-ResNet-110 & 92.28 & \textbf{72.67} & \textbf{55.24} & \textbf{32.05} & \textbf{31.74} \\
        In-ResNet-sig-110 & 93.49 & 55.04 & 44.65 & 6.29 & 5.94 \\
        In-ResNet-gating-110 & 93.46 & 54.53 & 41.25 & 5.65 & 5.33 \\
        In-ResNet-gating-sig-110 & 90.68 & 68.04 & 46.17 & 21.89 & 21.65 \\
        \hline
    \end{tabular}
    \caption{Accuracy and robustness of In-ResNet-110, In-ResNet-sig-110, In-ResNet-gating-110, and In-ResNet-gating-sig-110 models, as well as the ResNet-110 baseline on CIFAR-10 benchmarks. ``Acc." denotes the accuracy over CIFAR-10 testing set. ``noise" denotes the \textbf{average} accuracy of the four stochastic noise groups from CIFAR-10-C. ``FGSM", ``IFGSM", and ``PGD" represent accuracy under the corresponding attacks at the attack radius of $2/255$. All of the results reported are averaged over 5 runs.}
    \label{comp-results}
\end{table*}

\textbf{Loss landscape analysis} As is given by the Lyapunov analysis, the robustness improvement is theoretically provided in that the damped models enjoy more locally stable points than the original ones. To further verify this, we visualize the loss landscapes of In-ResNet-110 and ResNet-110 models trained on CIFAR-10 benchmark along the attack direction. For a instance $(\bx, y)$, we plot the loss function $L(\bx, y)$ of along the FGSM attack direction. We also select a random orthogonal direction from the FGSM attack one and plot the model predictions of each grids. The unit of each axis in the figures is at the scale of $1/255$. To better analyze model robustness, we select the data instance $(\bx, y)$ for the CIFAR-10-C dataset, namely $(\bx^{'}, y)$, where $\bx^{'}$ is $\bx$ with injected stochastic noise.

Figure \ref{landscape} illustrates the loss landscapes of the ResNet-110 and In-ResNet-110 models along the FGSM attack direction. We select two input data instances: for Figure \ref{landscape}-\{(a)-(c)\}, the input is the $3$-th image in the shot noise group of the CIFAR-10-C dataset, the ground-truth label of which is ship; for Figure \ref{landscape}-\{(d)-(f)\}, the input is the $8$-th image in the speckle noise group of the CIFAR-10-C dataset, the ground-truth label of which is horse. For the first input example, ResNet-110 and In-ResNet-110 both make the correct prediction; for the second input example, they both make the wrong one. It can be seen that the added damping term have damped the loss landscape along the FGSM attack direction, resulting in a much weaker amplitude (the first example), or even turned the amplifying loss landscape of ResNet-110 into a damping one of In-ResNet-110 (the second example). Whether ResNet-110 and In-ResNet-110 both make the correct prediction or the wrong one, it is clear that the In-ResNet-110 model enjoys better robustness than ResNet-110, which agrees with our Lyapunov analysis that the damping term has introduced more locally stable points.
\vspace{-10pt}
\begin{figure*}[tbp]
\centering
\subfigure[Loss landscape.]{
    \includegraphics[scale=0.055]{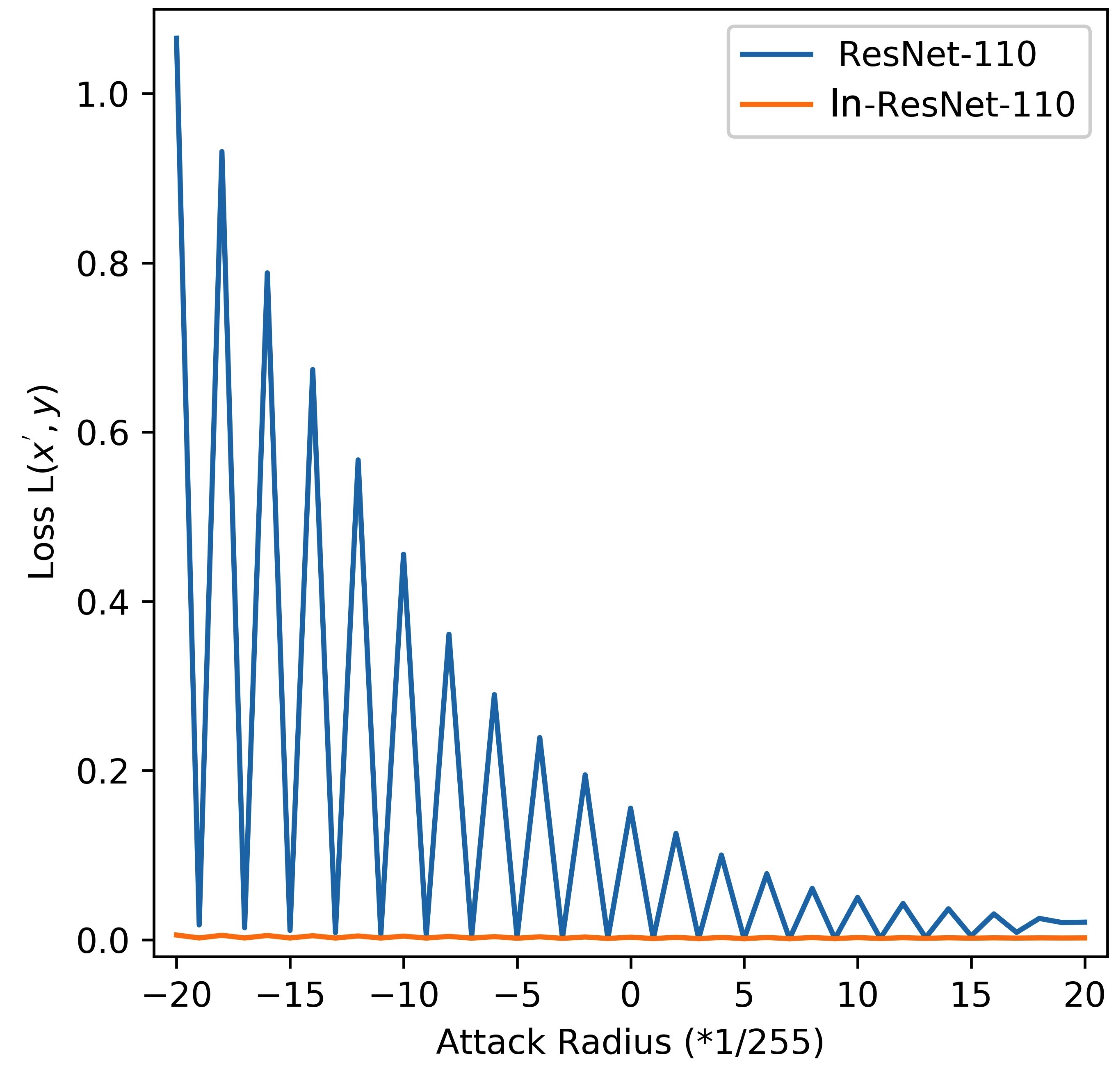}
}
\subfigure[ResNet-110 predictions]{
    \includegraphics[scale=0.22]{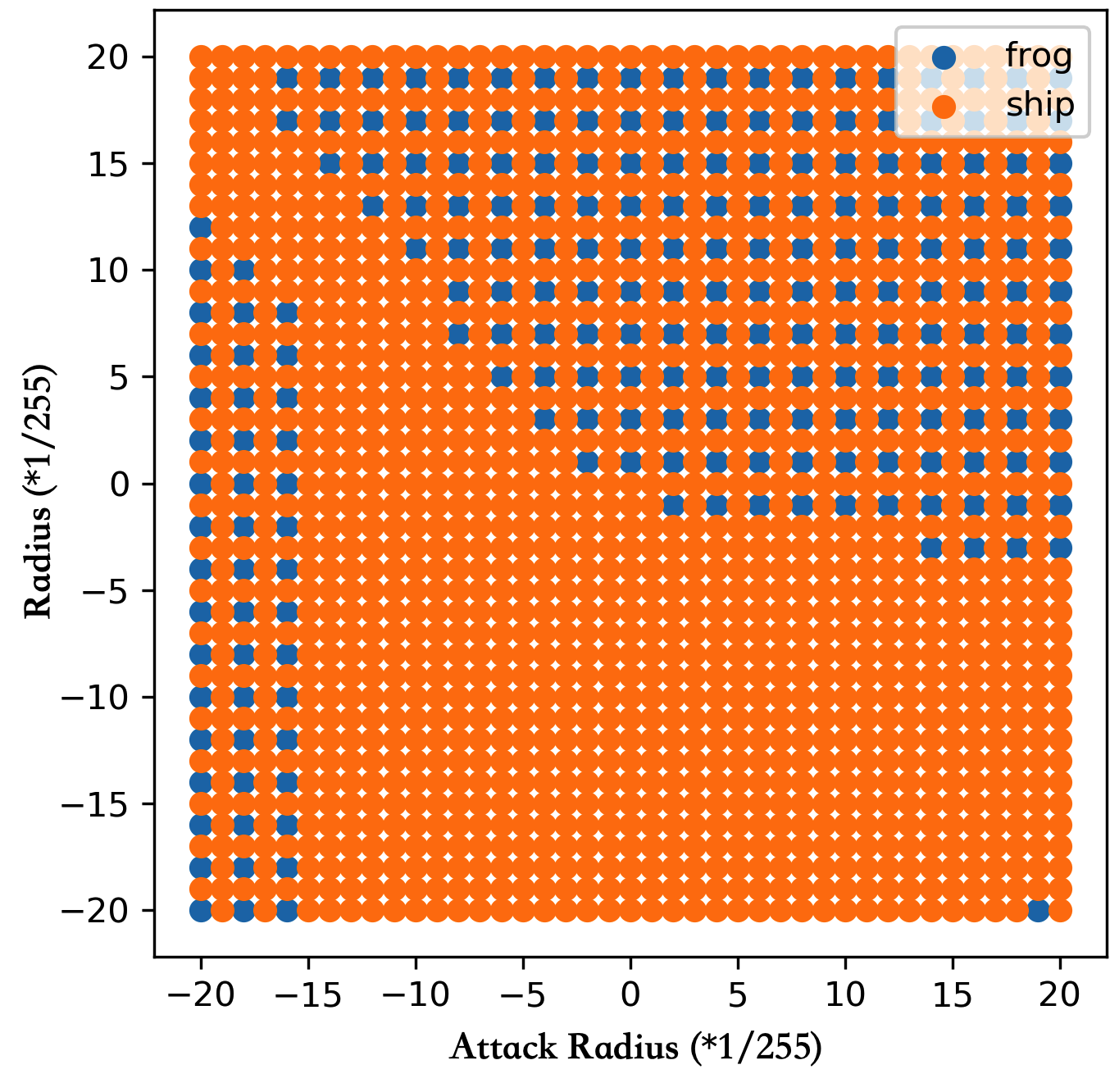}
}
\subfigure[In-ResNet-110 predictions]{
    \includegraphics[scale=0.22]{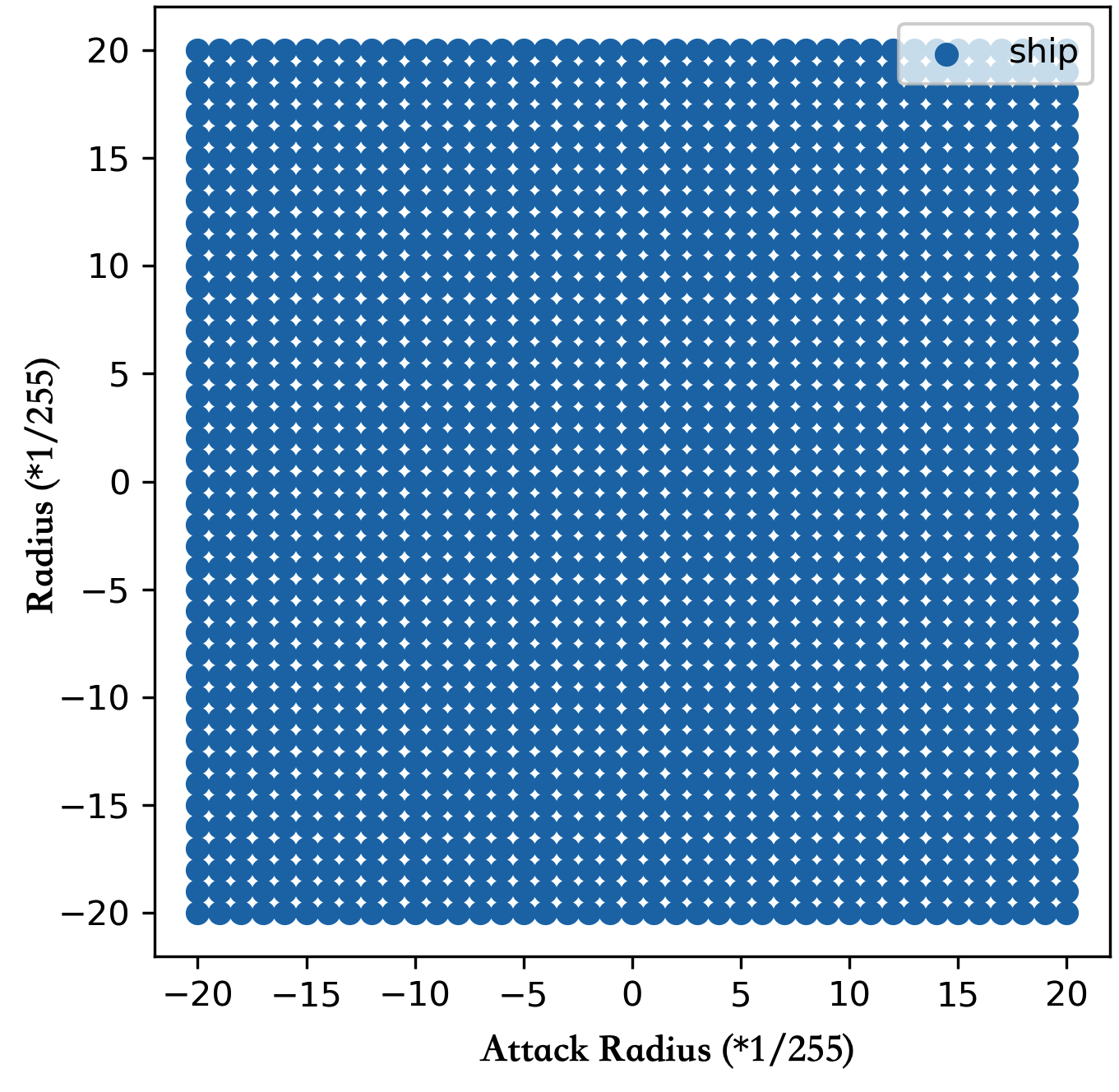}
}
\quad
\subfigure[Loss landscape.]{
    \includegraphics[scale=0.055]{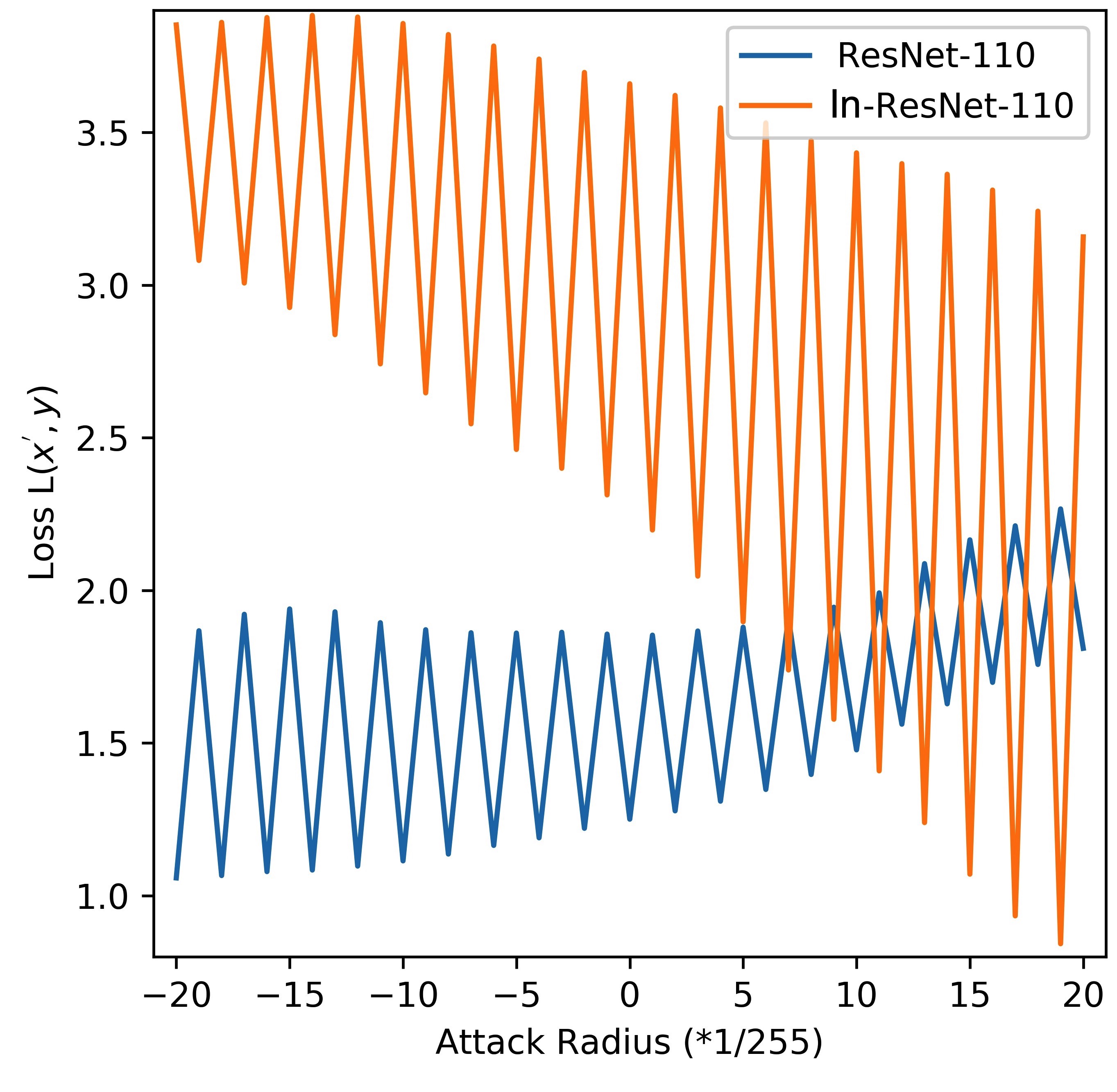}
}
\subfigure[ResNet-110 predictions]{
    \includegraphics[scale=0.22]{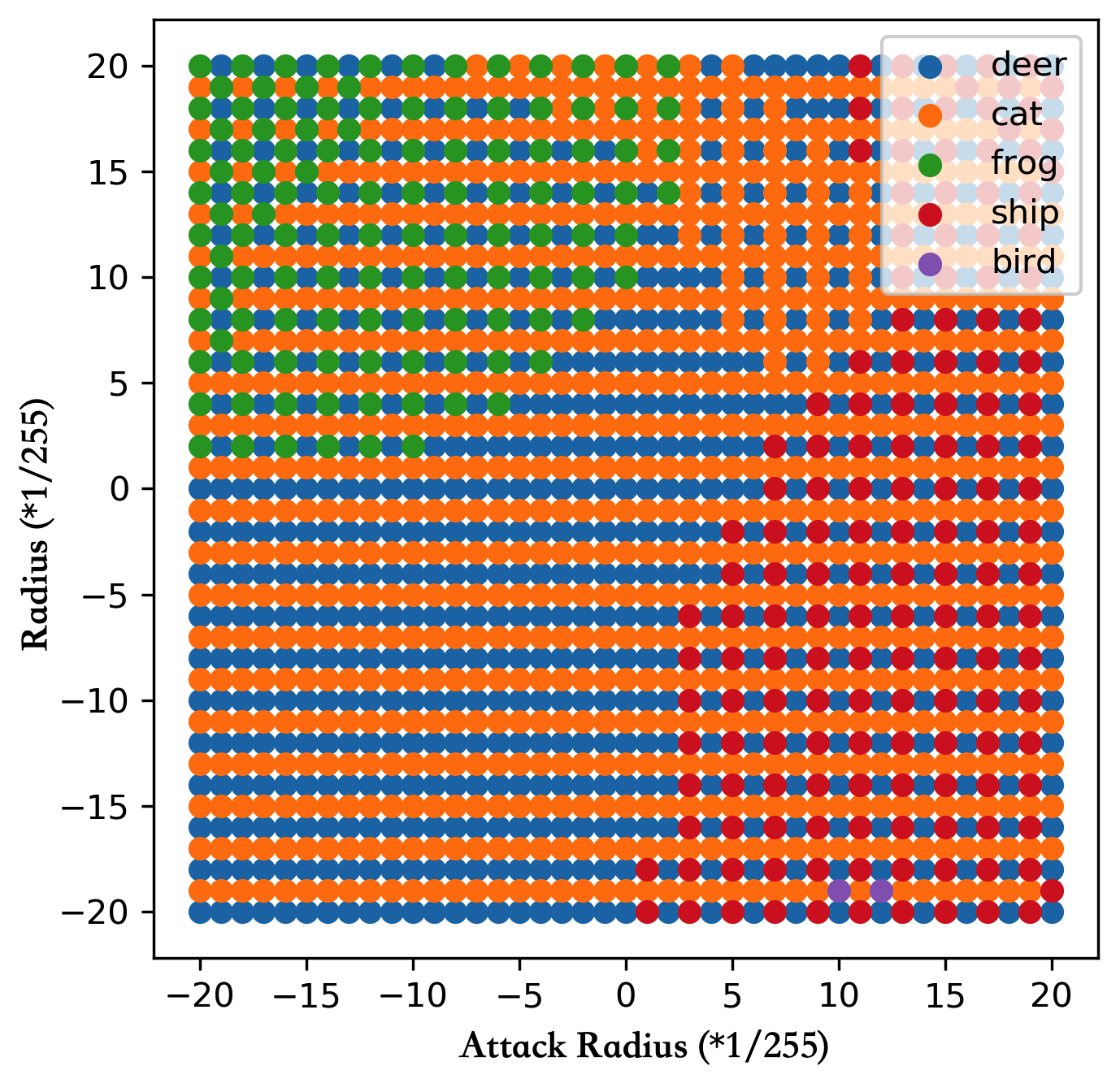}
}
\subfigure[In-ResNet-110 predictions]{
    \includegraphics[scale=0.22]{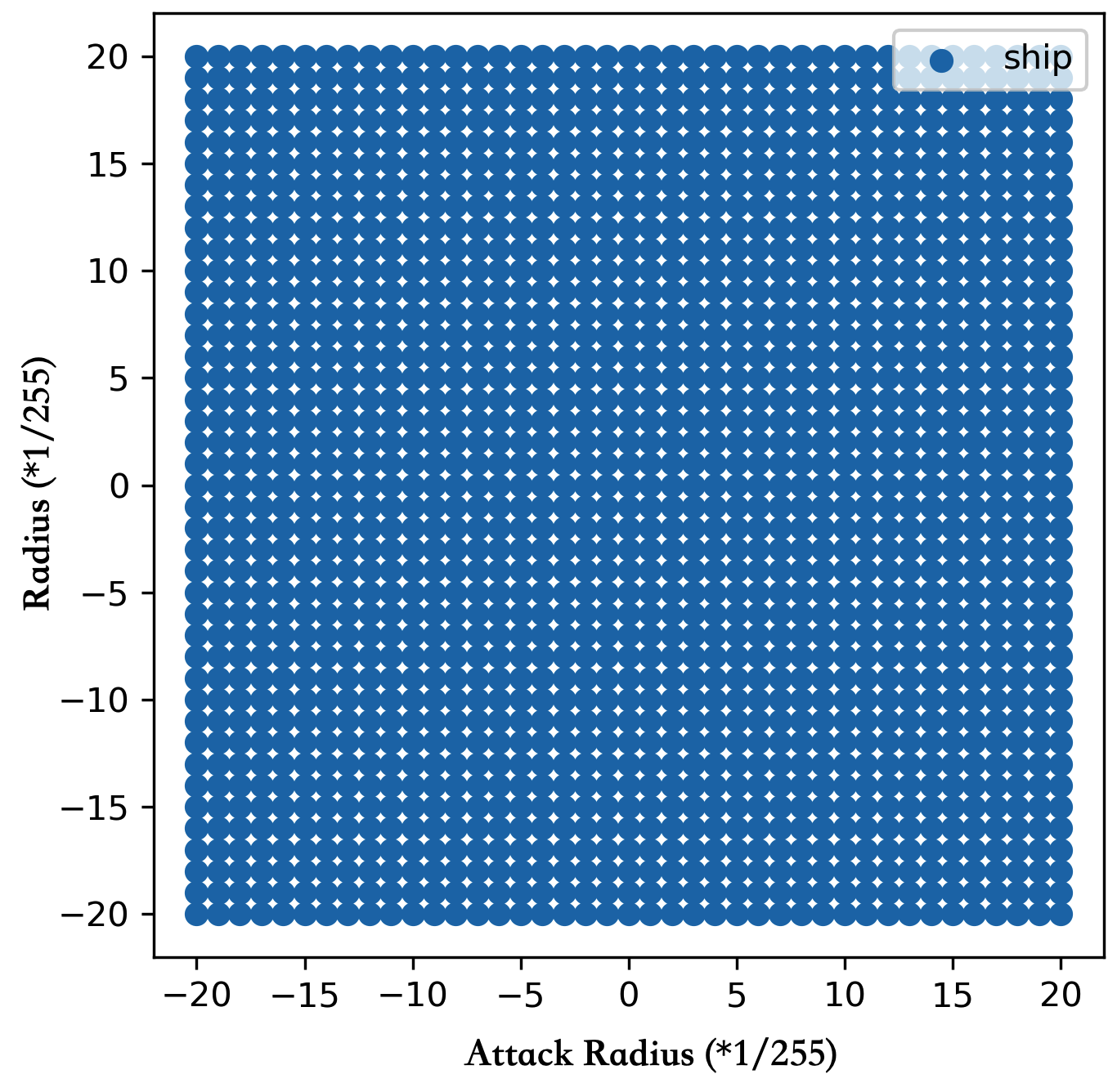}
}
\caption{ The input data instance is the $3$-th/$8$-th image in the shot/speckle noise group of the CIFAR-10-C dataset for (a)-(c)/(d)-(f), the ground truth label of which is ship/horse. For (a)-(c)/(d)-(f), ResNet-110 and In-ResNet-110 both make the correct / wrong prediction. (a) and (c) depict the loss landscape of ResNet-110 and In-ResNet-110 along the FGSM attack direction. \{(b) and (e)\} / \{(c) and (f)\} illustrate model predictions of \{ResNet-110\} / \{In-ResNet-110\} at each grids determined by the FGSM attack direction and a random orthogonal direction.}
\label{landscape}
\end{figure*}
\begin{table*}[tbp]
    \centering
    \begin{tabular}{l|c|r|r|r|r|r}
        \hline
        Model & Initialization & Acc. & noise & FGSM & IFGSM & PGD \\
        \hline
        \hline
        ResNet & - & \textbf{93.58} & 53.70 & 41.48 & 5.93 & 5.60 \\
        \hline
         & $\mathcal{U}[0.00, 0.10]$ & \textbf{93.51} & 55.15 & 46.74 & 8.39 & 7.96 \\
         & $\mathcal{U}[0.10, 0.20]$ & 93.25 & 62.88 & 49.58 & 16.89 & 16.46 \\
        In-ResNet &$\mathcal{U}[0.20, 0.25]$ & 92.28 & 72.67 & 55.24 & 32.05 & 31.74 \\
         &$\mathcal{U}[0.25, 0.30]$ & 91.63 & 76.20 & 55.79 & 36.53 & 36.28 \\
        &$\mathcal{U}[0.30, 0.40]$  & 90.62 & \textbf{79.35} & \textbf{55.95} & \textbf{41.07} & \textbf{40.84} \\
        \hline           
         & $\mathcal{U}[0.00, 0.10]$ & \textbf{93.41} & 54.18 & 42.28 & 6.78 & 6.48 \\
         & $\mathcal{U}[0.10, 0.20]$ & 92.86 & 63.58 & 46.07 & 16.99 & 16.60 \\
        $\lambda$-In-ResNet &$\mathcal{U}[0.20, 0.25]$ & 92.15 & 72.35 & 50.84 & 30.72 & 30.45 \\        
         &$\mathcal{U}[0.25, 0.30]$ & 91.30 & 75.65 & 53.29 & 36.90 & 36.74 \\
        &$\mathcal{U}[0.30, 0.40]$  & 90.17 & \textbf{79.66} & \textbf{55.03} & \textbf{41.06} & \textbf{40.94} \\
        \hline         
    \end{tabular}
    \caption{Accuracy and robustness results of In-ResNet-110 and $\lambda$-In-ResNet-110 with different initialization schemes. ``Acc." denotes the accuracy over CIFAR-10 testing set. ``noise" denotes the \textbf{average} accuracy of the four stochastic noise groups from CIFAR-10-C. ``FGSM", ``IFGSM", and ``PGD" represent model accuracy under the corresponding attacks at the radius of $2/255$. All of the results reported are averaged over 5 runs except for $\mathcal{U}[0.3, 0.4]$: they are averaged over 4(2) runs, as 1(3) out of 5 runs for In-ResNet-110 ($\lambda$-In-ResNet-110) failed with a final accuracy of 10\% on CIFAR-10 test set.}
    \label{init-results}
\end{table*}

\subsection{Comparison among In-ResNet Variants}
While Eq. (\ref{In-ResNet1}) depicts the In-ResNet structure, in this section, we propose several variants of In-ResNet and compare their performances. To facilitate the discussion, the In-ResNet can be written in the general form:
\begin{equation}
    \bx_{n+1} = (1 - {\rm act}(d(\bx_{n})))\bx_n + \Delta t f_n(\bx_n),
\end{equation}
\noindent where $d(x_{n})$ is the function determining the interpolation coefficients. ${\rm act}$ is the activation function. For In-ResNet, the $d(x_{n})$ is a learnable scalar parameter $\lambda_n$; ${\rm act}$ is ${\rm ReLU}$ function. We propose several In-ResNet variants:

\begin{itemize}
    \item $d(\bx_{n}) = \lambda_n$, ${\rm act} = {\rm sigmoid}$: we replace the activation function to be ${\rm sigmoid}$, which restricts the interpolation coefficients to be within $[0, \; 1]$, and thus guarantees that the learned model is an interpolation. We refer to it as In-ResNet-sig.
    \item $d(\bx_{n}) = W_d \bx_n + b_d$, ${\rm act} = {\rm ReLU}$: we let the learnable scalar parameters determined by a linear transformation from input $x_n$, yielding a gating mechanism. We refer to it as In-ResNet-gating.
    \item $d(\bx_{n}) = W_d \bx_n + b_d$, ${\rm act} = {\rm sigmoid}$: based on the previous variant, we further replace the activation function to be ${\rm sigmoid}$. It is noteworthy that this variant is the shortcut-only gating mechanism discussed in \cite{resnet}. We refer to it as In-ResNet-gating-sig.
\end{itemize}

We use In-ResNet-110 as the basic In-ResNet model and experiment on CIFAR-10 benchmark to compare their performance. The accuracy and robustness results are reported averagely from 5 runs, shown in Table \ref{comp-results}.  We elaborately tune the initialization intervals and report the model with the largest sum of the accuracy over both the CIFAR-10 testing set and the noise groups in the CIFAR-10-C dataset. 

It can be seen that In-ResNet-110 leads to the largest robustness improvements over ResNet-110 baseline, with a relatively small accuracy drop. The In-ResNet-sig-110 model achieves better accuracy result than In-ResNet-110, however, its performance on robustness improvements are marginal. This is because the learned interpolation coefficients in In-ResNet-sig-110 are close to $0$, resulting in nearly identity skip-connections. Similarly, the performance of In-ResNet-gating-110 is very close to ResNet-110 baseline due to the degeneration of its damped skip-connections. The In-ResNet-gating-sig-110 model also improves over the ResNet-110 baseline with a large margin in terms of robustness performance. The improvement, however, is less significant than our In-ResNet-110 model. The accuracy of the In-ResNet-gating-sig-110 model also lags behind In-ResNet-110, which may attribute to the extra optimization difficulty introduced by the gating mechanism.

\begin{table*}[htbp]
    \centering
    \begin{tabular}{l|r|r|r|r|r}
        \hline
        Model & Acc. & noise & FGSM & IFGSM & PGD \\
        \hline
        \hline
        ResNet-110 & 93.58 & 53.70 & 41.48 & 5.93 & 5.60 \\
        ResNet-110, ens & \textbf{95.03} & 55.70 & 43.99 & 6.26 & 5.93 \\ 
        In-ResNet-110 & 92.28 & 72.67 & 55.24 & 32.05 & 31.74 \\
        In-ResNet-110, ens & 94.03 & \textbf{75.86} & \textbf{58.42} & \textbf{34.44} & \textbf{34.03} \\
        $\lambda$-In-ResNet-110 & 92.15 & 72.35 & 50.84 & 30.72 & 30.45 \\
        $\lambda$-In-ResNet-110, ens & 94.00 & 75.29 & 53.66 & 32.95 & 32.77 \\
        \hline 
        ResNet-164 & 94.46 & 56.51 & 44.37 & 8.19 & 7.77 \\
        ResNet-164, ens & \textbf{95.44} & 58.76 & 46.54 & 8.53 & 8.14 \\ 
        In-ResNet-164 & 92.69 & 72.05 & 51.84 & 27.43 & 26.95 \\
        In-ResNet-164, ens & 94.26 & \textbf{75.26} & \textbf{54.72} & \textbf{28.97} & \textbf{28.51} \\
        $\lambda$-In-ResNet-164 & 92.55 & 71.88 & 50.53 & 26.50 & 26.04 \\
        $\lambda$-In-ResNet-164, ens & 94.20 & 74.97 & 53.17 & 27.74 & 27.30 \\      
        \hline
    \end{tabular}
    \caption{Comparison between the accuracy and robustness results of the ensemble model over 5 different runs and those of the single model (scores are averaged). ``Acc." denotes the accuracy over CIFAR-10 testing set. ``noise" denotes the \textbf{average} accuracy of the four stochastic noise groups from CIFAR-10-C. ``FGSM", ``IFGSM", and ``PGD" represent model accuracy under the corresponding attacks at the radius of $2/255$. }
    \label{ensemble}
\end{table*}

\subsection{Trade-off between Optimization and Robustness}\label{sec:analysis}
As is shown in Table \ref{acc}, while ($\lambda$-)In-ResNet enjoys better robustness, it suffers from optimization difficulty: an accuracy degeneration around 2\% is caused by our ($\lambda$-)In-ResNet model. In this section, we show that the initialization of $\lambda_n$ is of great importance to the optimization process. We use In-ResNet-110 model and $\lambda$-In-ResNet-110 model trained on CIFAR-10 benchmark as the basic model, initializing $\lambda_n$ by randomly sampling from $\mathcal{U}[x, y]$. For the basic model, we have that $\mathcal{U}[x, y] = \mathcal{U}[0.2, 0.25]$. We try the following initialization schemes as well: $\mathcal{U}[x, y] = \mathcal{U}[0, 0.1]$, $\mathcal{U}[0.1, 0.2]$,  $\mathcal{U}[0.2, 0.25]$,  $\mathcal{U}[0.25, 0.3]$, and $\mathcal{U}[0.3, 0.4]$. The accuracy and robustness results are reported averagely from 5 runs, shown in Table \ref{init-results}.

From the experimental results, we can see that the performance of ($\lambda$-)In-ResNet is sensitive to the initialization of $\lambda_n$. On one hand, as the initialization becomes larger, the model robustness goes up. This agrees with our Lyapunov analysis, as the larger initialization of $\lambda_n$'s tends to help model to converge to the larger final $\lambda_n$'s, yielding larger damping terms and better robustness. One the other hand, larger initialization leads to worse accuracy results. Especially for $\mathcal{U}[0.30, 0.40]$, 1(3) out of 5 runs of In-ResNet-110 ($\lambda$-In-ResNet-110) fails the optimization with a final accuracy of $10 \%$. This can be interpreted that the damped shortcuts hamper information propagation and lead to optimization difficulty \cite{resnet}. More results on CIFAR-100 benchmark can be found in Appendix B.

\subsection{Effect of Model Ensemble}

\begin{figure}
    \centering
    \includegraphics[scale=0.06]{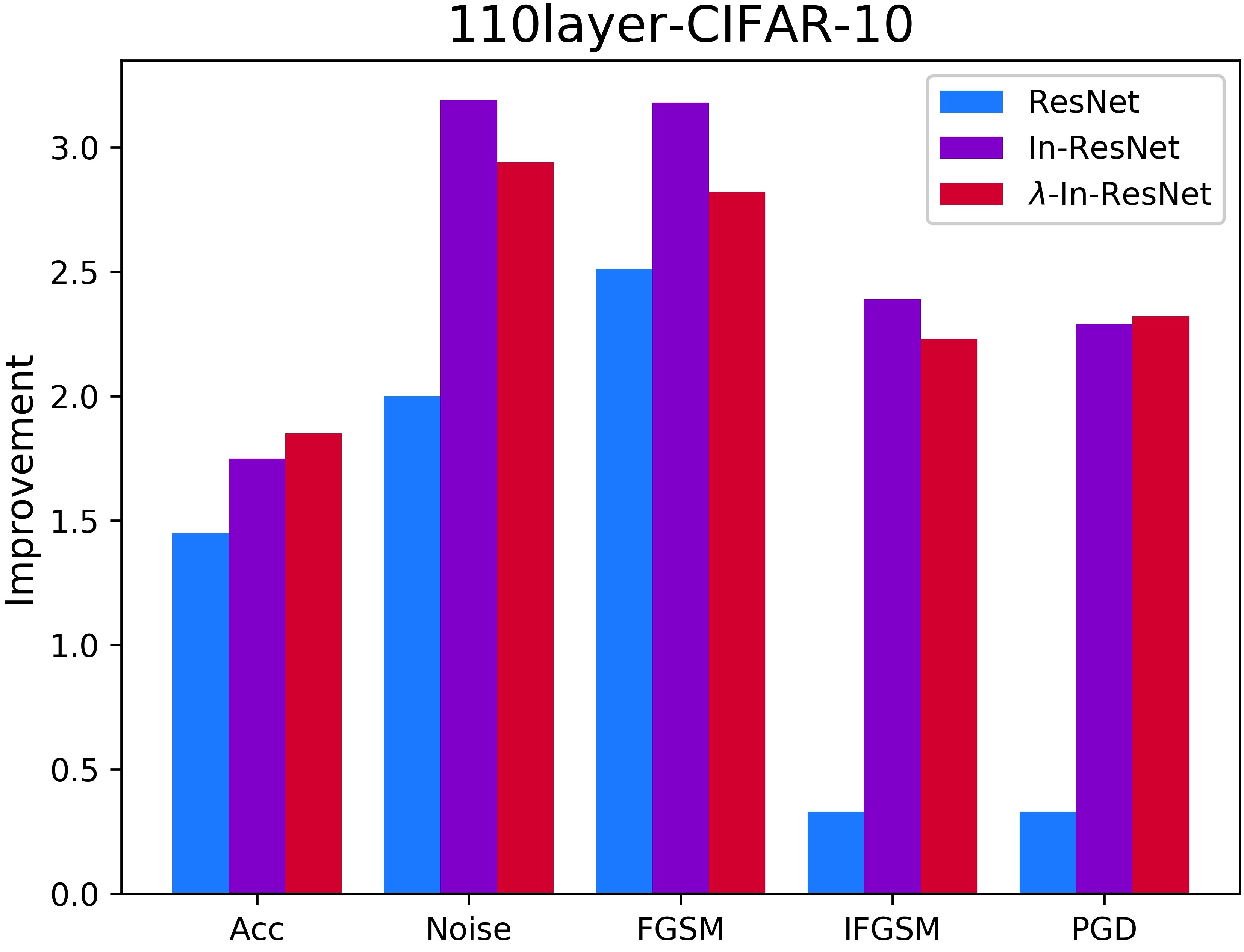}
    \caption{The accuracy \textbf{improvements} over single models for the ensemble ResNet-110, In-ResNet-110 and $\lambda$-In-ResNet-110 over CIFAR-10 dataset. Both of the ensemble of our models have more significant accuracy improvements than the ensemble of the baseline ResNet-110 model.}
    \label{imp}
\end{figure}

It is known that an ensemble model is more robust than a single model \cite{EnResNet}. To further improve accuracy and robustness, we perform model ensemble over the 5 different runs of baseline and our models. Table \ref{ensemble} shows the comparison between ensemble models and single models for ResNet-110, In-ResNet-110 and $\lambda$-ResNet-110 over CIFAR-10 dataset. It can be seen that all of the ensemble models are more robust and more accurate than the corresponding single models. 

We also plot the accuracy \textbf{improvements} over single models for the ensemble ResNet-110, In-ResNet-110 and $\lambda$-In-ResNet-110. As shown in \ref{imp}, both of the ensemble of our models have more significant accuracy improvements than the ensemble of the baseline ResNet-110 model. This can be attributed to the performance difference among different runs of our model due to optimization difficulty. More results and visualizations of the effect of ensemble method can be found in Appendix C.

\section{Conclusion}\label{sec:con}

While the relationship between ODEs and non-residual networks remains unclear, in this paper, we present a novel ODE model by adding a damping term. By adjusting the interpolation coefficient, the proposed model unifies the interpretation of both residual and non-residual networks. Lyapunov analysis and experimental results on CIFAR-10 and CIFAR-100 benchmarks reveals better robustness of the proposed interpolated networks against both stochastic noise and several adversarial attack methods. Loss landscape analysis reveals the improved robustness of our method along the attack direction. Furthermore, experiments show that the performance of proposed model is sensitive to the initialization of the interpolation coefficients, demonstrating trade-off between optimization difficulty and robustness. The significance of the design of interpolated networks is shown by comparing several model variants. Future work includes determining the interpolated coefficients as a black-box process and leveraging data augmentation techniques to improve our models.

\section*{Acknowledgements}

We thank all the anonymous reviewers for their suggestions. Yang Liu is supported by the National Key R\&D Program of China (No. 2017YFB0202204), National Natural Science Foundation of China (No. 61925601, No. 61761166008), and Huawei Technologies Group Co., Ltd. Chenglong Bao is supported by National Natural Sciences Foundation of China (No. 11901338) and Tsinghua University Initiative Scientific Research Program. Zuoqiang Shi is supported by National Natural Sciences Foundation of China (No. 11671005). This work is also supported by Beijing Academy of Artificial Intelligence.

\bibliography{interpolation}
\bibliographystyle{icml2020}

\onecolumn

\appendix

\section{Visualization of Learned Interpolation Coefficients}\label{appendix_a}

Here we provide more visualizations of learned interpolation coefficients in ($\lambda$-)In-ResNet-110 and ($\lambda$-)In-ResNet-164 trained on CIFAR-10 and CIFAR-100 dataset. Figure \ref{learned_coeffs-10-lamin} illustrates the coefficients in $\lambda$-In-ResNet-110 and $\lambda$-In-ResNet-164 models trained on CIFAR-10 benchmarks, and Figure \ref{learned_coeffs-100-in} for In-ResNet-110 and In-ResNet-164 on CIFAR-100, Figure \ref{learned_coeffs-100-lamin} for $\lambda$-In-ResNet-110 and $\lambda$-In-ResNet-164 on CIFAR-100.

\begin{figure}[htbp]
    \centering
    \includegraphics[scale=0.172]{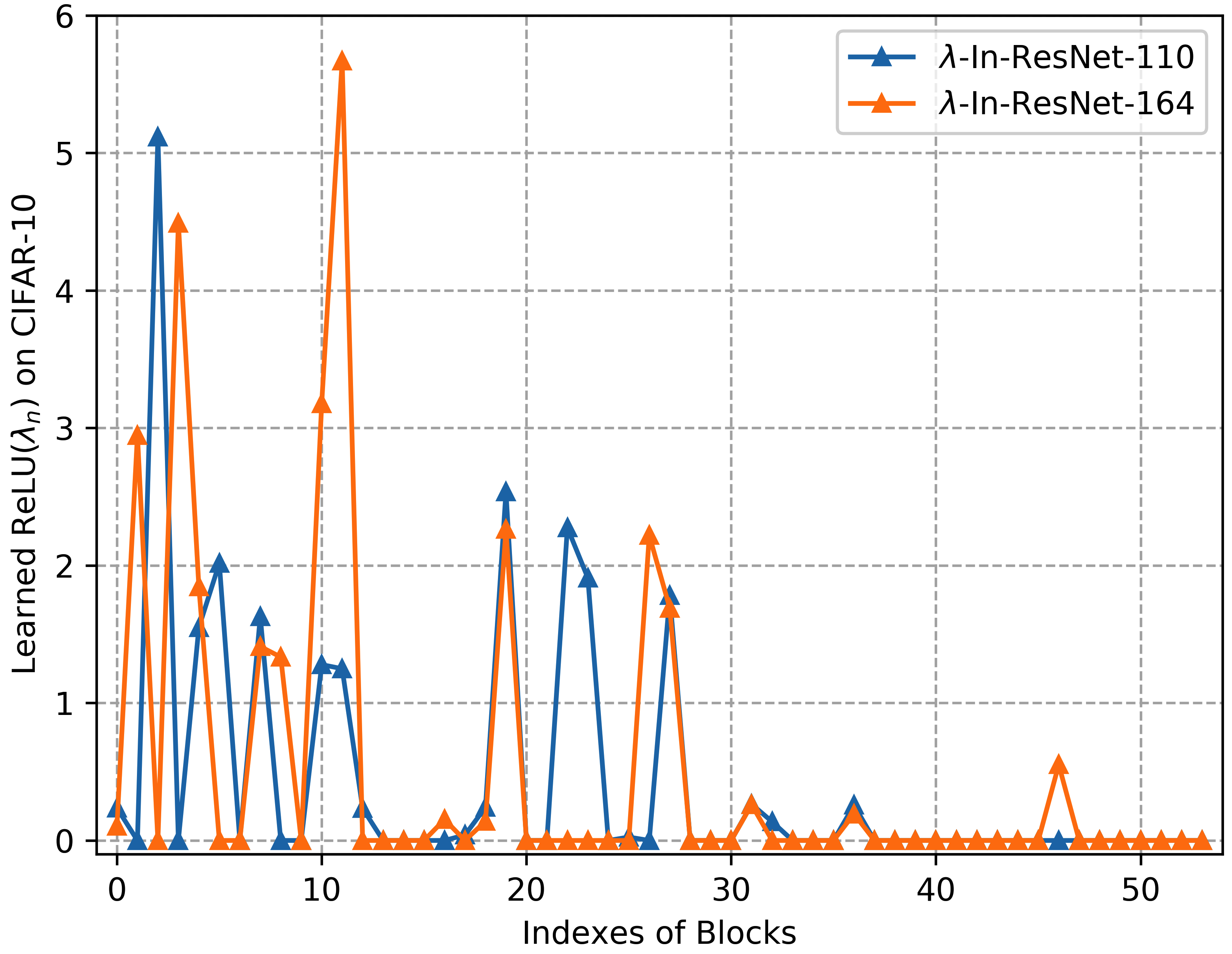}
    \caption{Learned interpolation coefficients in $\lambda$-In-ResNet-110 and $\lambda$-In-ResNet-164 models trained on CIFAR-10 benchmarks. }
    \label{learned_coeffs-10-lamin}
\end{figure}

\

\begin{figure}[htbp]
    \centering
    \includegraphics[scale=0.172]{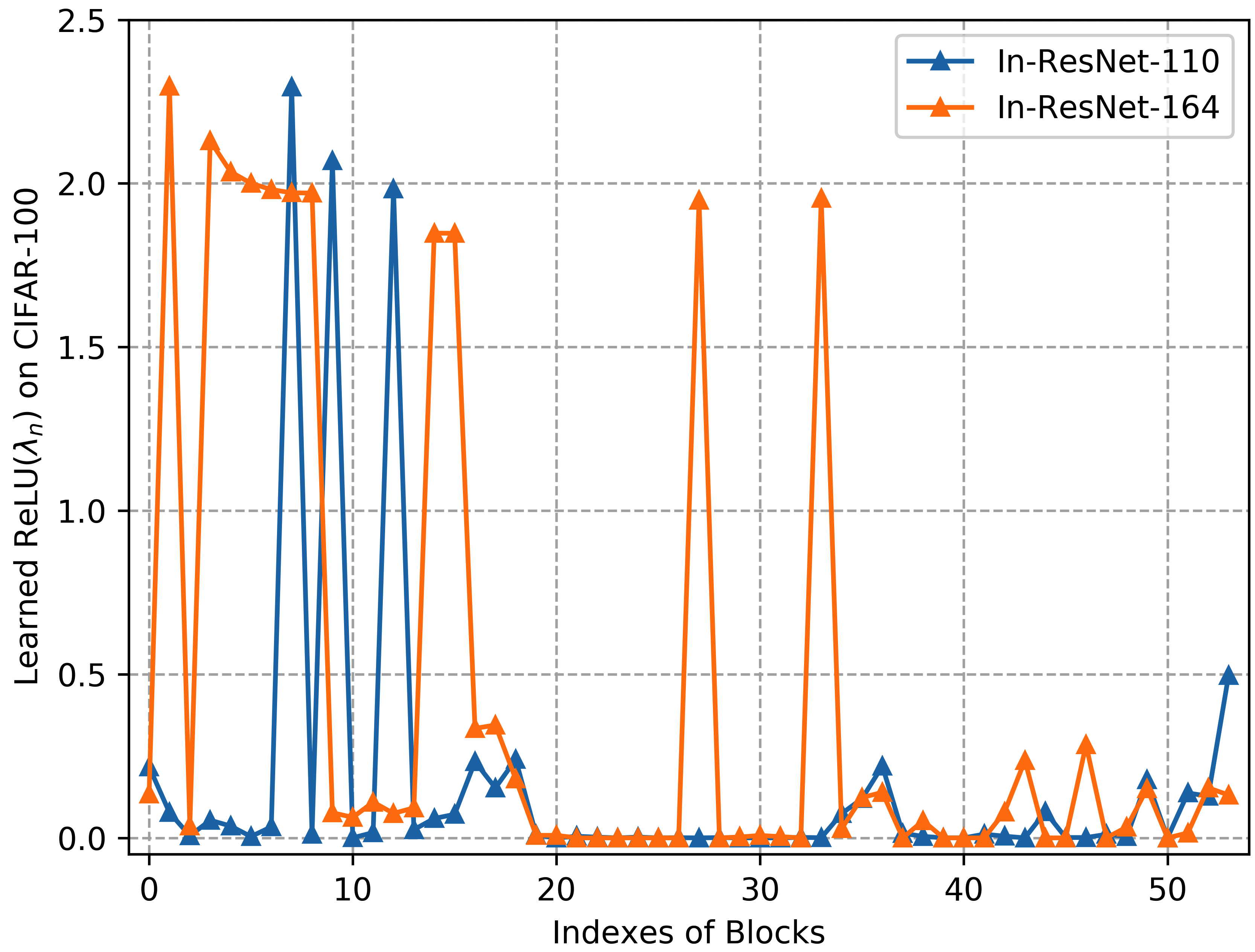}
    \caption{Learned interpolation coefficients in In-ResNet-110 and In-ResNet-164 models trained on CIFAR-100 benchmarks. }
    \label{learned_coeffs-100-in}
\end{figure}

\begin{figure}[htbp]
    \centering
    \includegraphics[scale=0.172]{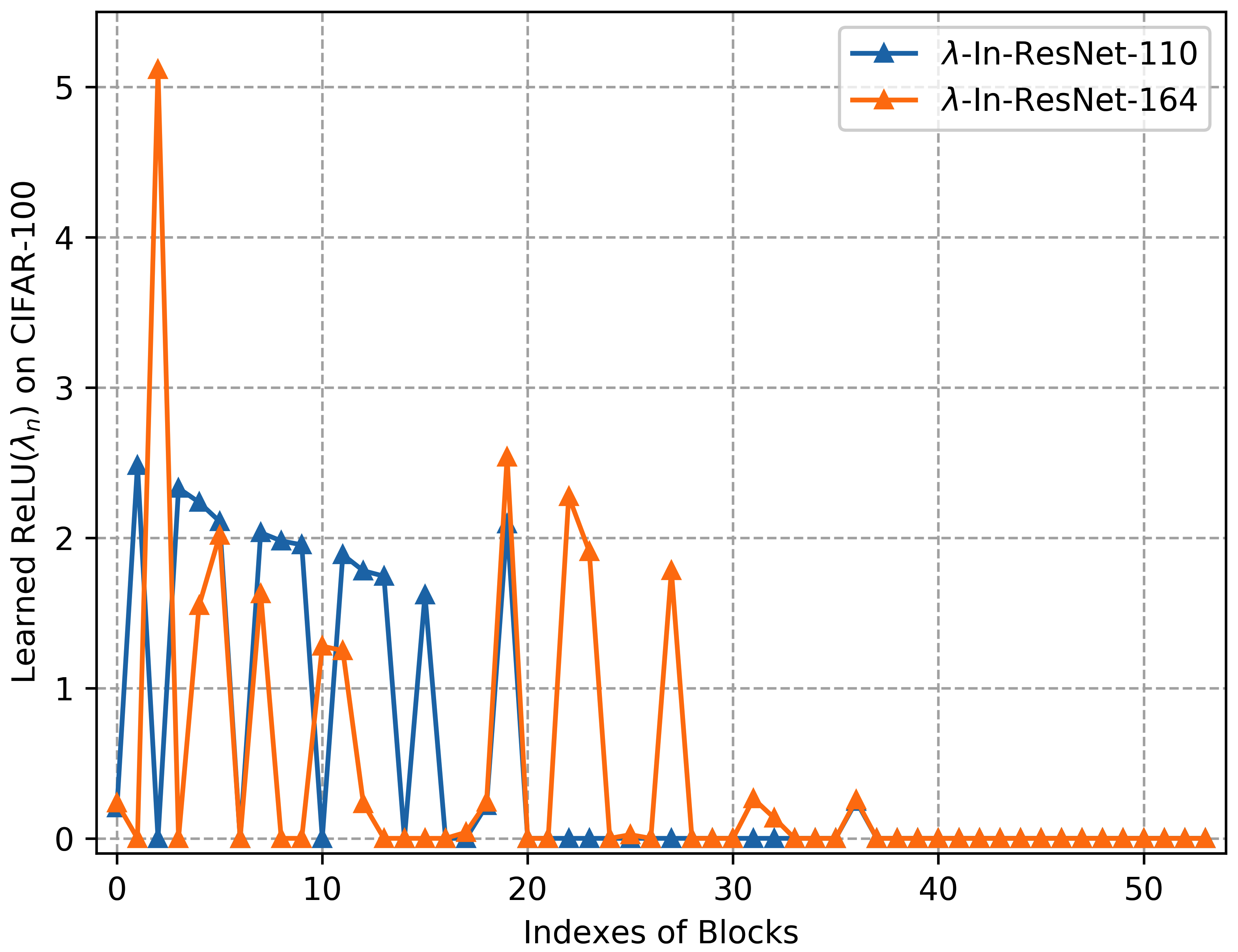}
    \caption{Learned interpolation coefficients in $\lambda$-In-ResNet-110 and $\lambda$-In-ResNet-164 models trained on CIFAR-100 benchmarks. }
    \label{learned_coeffs-100-lamin}
\end{figure}

\newpage

\section{Tradeoff between Optimization and Robustness - Results on CIFAR-100}\label{appendix_b}

\begin{table}[htbp]
    \centering
    \begin{tabular}{l|c|r|r|r|r|r}
        \hline
        Model & Initialization & Acc. & noise & FGSM & IFGSM & PGD \\
        \hline
        \hline
        ResNet & - & \textbf{72.73} & 25.76 & 18.74 & 2.18 & 2.11 \\
        \hline
         & $\mathcal{U}[0.00, 0.10]$ & \textbf{72.53} & 27.07 & \textbf{19.51} & 2.68 & 2.57  \\
         & $\mathcal{U}[0.10, 0.20]$ & 71.02 & 32.24 & 19.30 & 4.60 & 4.38 \\
        In-ResNet &$\mathcal{U}[0.20, 0.25]$ & 70.55 & 34.63 & 18.74 & 4.92 & 4.81 \\        
         &$\mathcal{U}[0.25, 0.30]$ & 69.30 & 37.90 & 18.96 & 6.97 & 6.86 \\
        &$\mathcal{U}[0.30, 0.40]$  & 68.13 & \textbf{39.26} & \textbf{19.47} & \textbf{8.03} & \textbf{7.97} \\
        \hline           
         & $\mathcal{U}[0.00, 0.10]$ & \textbf{72.27} & 27.41 & \textbf{18.94} & 2.69 & 2.58  \\
         & $\mathcal{U}[0.10, 0.20]$ & 71.29 & 31.99 & 18.24 & 4.11 & 3.97 \\
        $\lambda$-In-ResNet &$\mathcal{U}[0.20, 0.25]$ & 70.39 & 34.69 & 18.40 & 5.17 & 5.00 \\        
         &$\mathcal{U}[0.25, 0.30]$ & 68.87 & 37.07 & 18.37 & 6.43 & 6.28 \\
        &$\mathcal{U}[0.30, 0.40]$  & 68.31 & \textbf{38.56} & \textbf{18.75} & \textbf{6.62} & \textbf{6.47} \\
        \hline         
    \end{tabular}
    \caption{Accuracy and robustness results of In-ResNet-110 and $\lambda$-In-ResNet-110 with different initialization schemes. ``Acc." denotes the accuracy over CIFAR-100 testing set. ``noise" denotes the \textbf{average} accuracy of the four stochastic noise groups from CIFAR-100-C. ``FGSM", ``IFGSM", and ``PGD" represent model accuracy under the corresponding attacks at the radius of $4/255$. All of the results reported are averaged over 5 runs except for $\mathcal{U}[0.3, 0.4]$: they are averaged over 3(2) runs, as 2(3) out of 5 runs for In-ResNet-110 ($\lambda$-In-ResNet-110) failed with a final accuracy of 1\% on CIFAR-100 test set.}
    \label{init-results-100}
\end{table}

 In Table \ref{init-results}, we discussed about accuracy and robustness results of In-ResNet-110 and $\lambda$-In-ResNet-110 with different initialization schemes on CIFAR-10. Here we provide similar analysis of In-ResNet-110 and $\lambda$-In-ResNet-110 on CIFAR-100. Table \ref{init-results-100} depicts same phenomenon as original Table 5 does: as the initialization of interpolation coefficients becomes larger, the model gradually becomes non-residual, accuracy drops and robustness rises.

\section{Effect of the Ensemble Method}\label{appendix_c}

Table \ref{ensemble-all} shows the comparison between the accuracy and robustness results of the ensemble model over 5 different runs and those of the single model with different layers and benchmarks. Figure \ref{imp-164-10} and Figure \ref{imp-100} illustrates the accuracy improvements over single models for the ensemble models over different benchmarks. It is shown that the accuracy improvements for the ensemble of our models are mostly more significant than those for the baseline ResNet models.

\begin{table*}[htbp]
    \centering
    \begin{tabular}{l|l|r|r|r|r|r}
        \hline
        Benchmark & Model & Acc. & noise & FGSM & IFGSM & PGD \\
        \hline
        \hline
        & ResNet-110 & 93.58 & 53.70 & 41.48 & 5.93 & 5.60 \\
        & ResNet-110, ens & \textbf{95.03} & 55.70 & 43.99 & 6.26 & 5.93 \\ 
        & In-ResNet-110 & 92.28 & 72.67 & 55.24 & 32.05 & 31.74 \\
        & In-ResNet-110, ens & 94.03 & \textbf{75.86} & \textbf{58.42} & \textbf{34.44} & \textbf{34.03} \\
        & $\lambda$-In-ResNet-110 & 92.15 & 72.35 & 50.84 & 30.72 & 30.45 \\
        CIFAR-10 & $\lambda$-In-ResNet-110, ens & 94.00 & 75.29 & 53.66 & 32.95 & 32.77 \\
        \cline{2-7}
        & ResNet-164 & 94.46 & 56.51 & 44.37 & 8.19 & 7.77 \\
        & ResNet-164, ens & \textbf{95.44} & 58.76 & 46.54 & 8.53 & 8.14 \\ 
        & In-ResNet-164 & 92.69 & 72.05 & 51.84 & 27.43 & 26.95 \\
        & In-ResNet-164, ens & 94.26 & \textbf{75.26} & \textbf{54.72} & \textbf{28.97} & \textbf{28.51} \\
        & $\lambda$-In-ResNet-164 & 92.55 & 71.88 & 50.53 & 26.50 & 26.04 \\
        & $\lambda$-In-ResNet-164, ens & 94.20 & 74.97 & 53.17 & 27.74 & 27.30 \\  
        \hline     
        & ResNet-110 & 72.73 & 25.76 & 18.74 & 2.18 & 2.11 \\
        & ResNet-110, ens & \textbf{78.84} & 30.05 & 21.43 & 2.83 & 2.81 \\ 
        & In-ResNet-110 & 70.55 & 34.63 & 18.74 & 4.92 & 4.81 \\
        & In-ResNet-110, ens & 76.91 & \textbf{40.69} & \textbf{21.73} & \textbf{6.90} & \textbf{6.37} \\ 
        & $\lambda$-In-ResNet-110 & 70.39 & 34.69 & 18.40 & 5.17 & 5.00 \\
        CIFAR-100 & $\lambda$-In-ResNet-110, ens & 76.61 & 40.14 & 20.79 & 6.41 & 6.35 \\    
        \cline{2-7}
        & ResNet-164 & 76.06 & 26.95 & 23.58 & 3.45 & 3.31 \\
        & ResNet-164, ens & \textbf{80.64} & 30.28 & \textbf{27.05} & 4.02 & 3.94 \\ 
        & In-ResNet-164 & 72.94 & 35.12 & 22.30 & 6.59 & 6.34 \\
        & In-ResNet-164, ens & 77.78 & \textbf{39.69} & 24.95 & 7.34 & 7.26 \\ 
        & $\lambda$-In-ResNet-164 & 73.22 & 34.58 & 22.50 & 6.64 & 6.46 \\
        & $\lambda$-In-ResNet-164, ens & 77.73 & 38.86 & 24.82 & \textbf{7.58} & \textbf{7.43} \\
        \hline 
    \end{tabular}
    \caption{Comparison between the accuracy and robustness results of the ensemble model over 5 different runs and those of the single model (scores are averaged) with different layers and benchmarks. ``Acc." denotes the accuracy over CIFAR-10 testing set. ``noise" denotes the \textbf{average} accuracy of the four stochastic noise groups from CIFAR-C. ``FGSM", ``IFGSM", and ``PGD" represent model accuracy under the corresponding attacks at the radius of $4/255$.}
    \label{ensemble-all}
\end{table*}

\begin{figure}[htbp]
    \centering
    \includegraphics[scale=0.07]{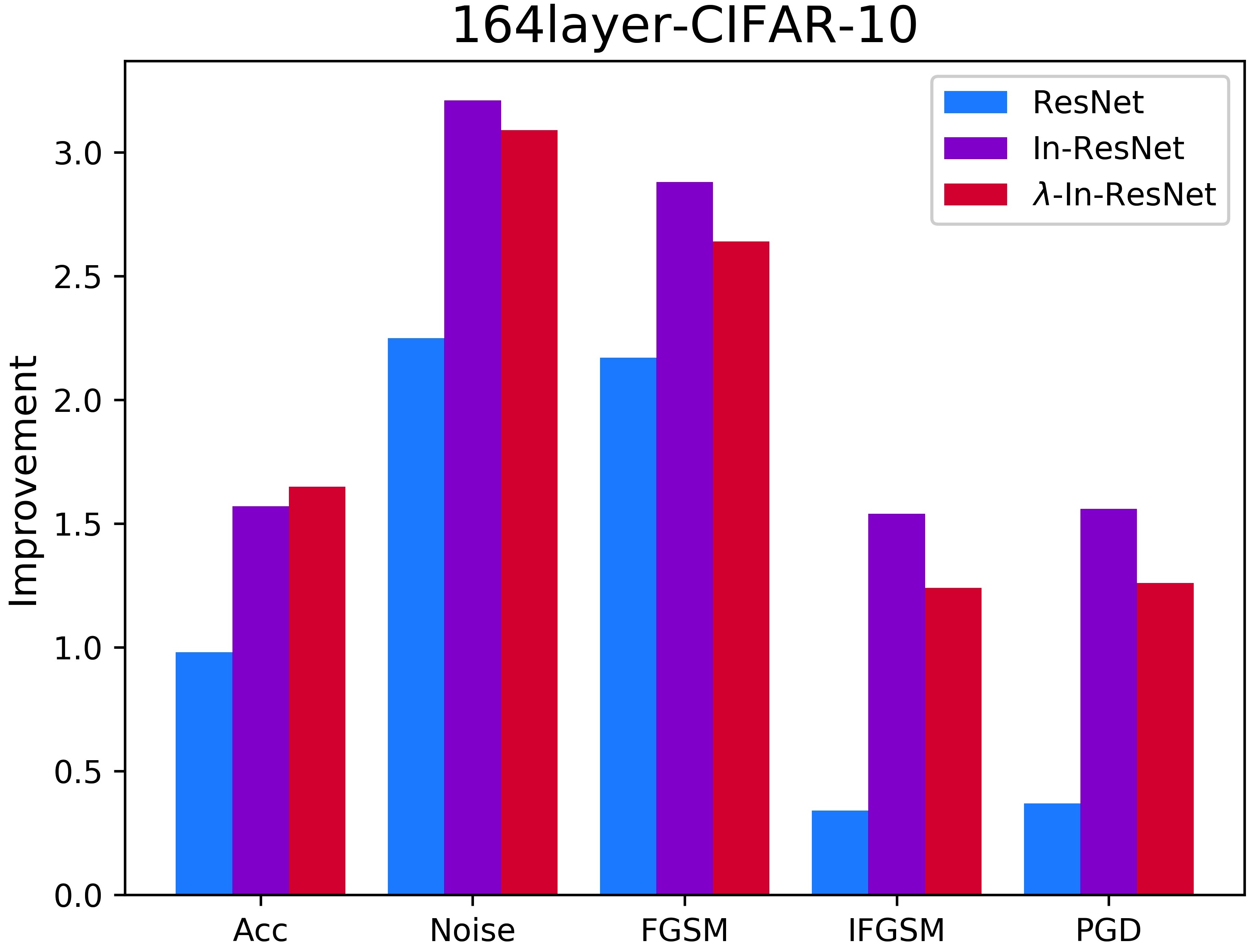}
    \caption{The accuracy \textbf{improvements} over single models for the ensemble ResNet-164, In-ResNet-164 and $\lambda$-In-ResNet-164 over CIFAR-10 dataset. Both of the ensemble of our models have more significant accuracy improvements than the ensemble of the baseline ResNet-164 model.}
    \label{imp-164-10}
\end{figure}

\begin{figure*}[tbp]
\centering
\subfigure[]{
    \includegraphics[scale=0.07]{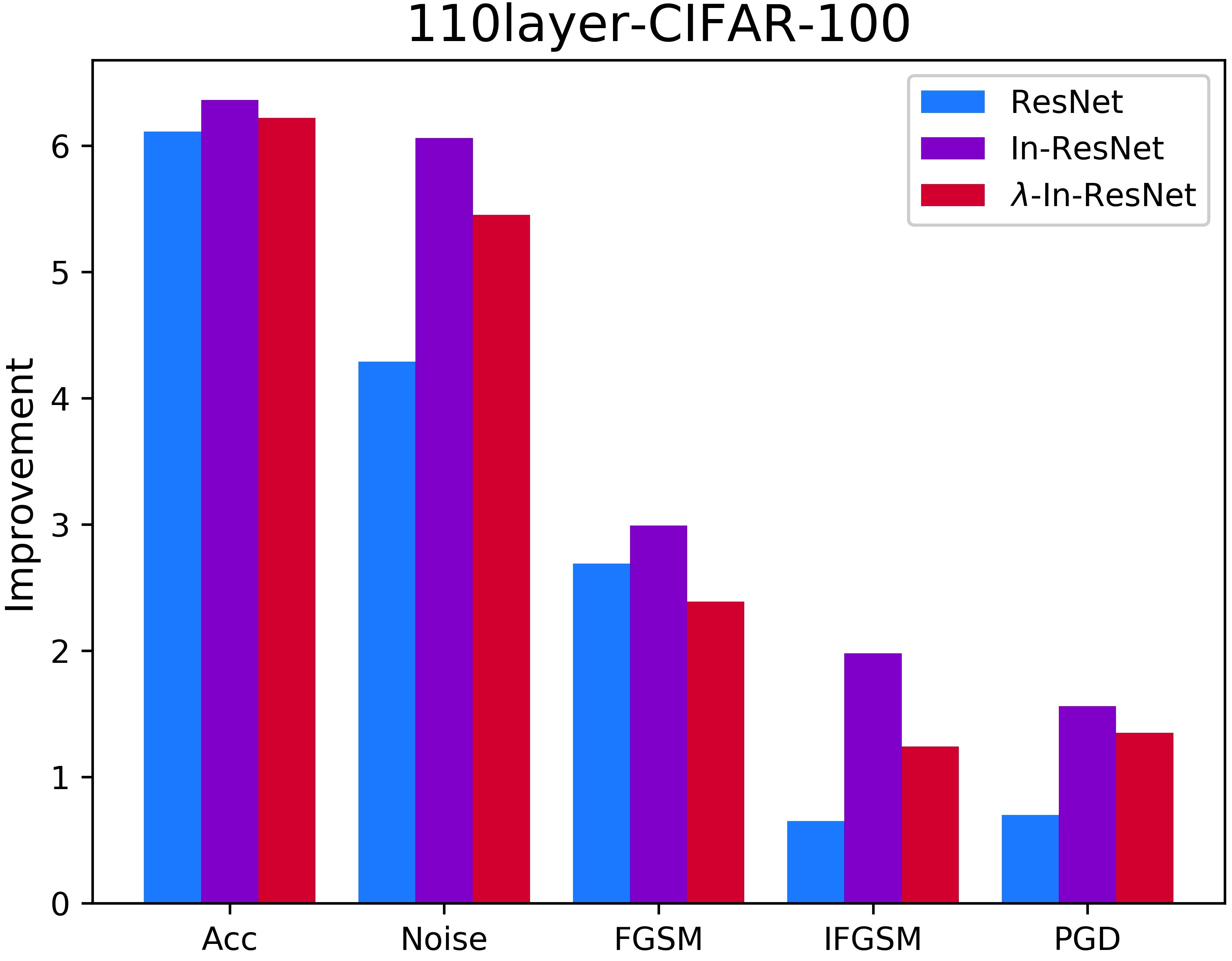}
}
\subfigure[]{
    \includegraphics[scale=0.07]{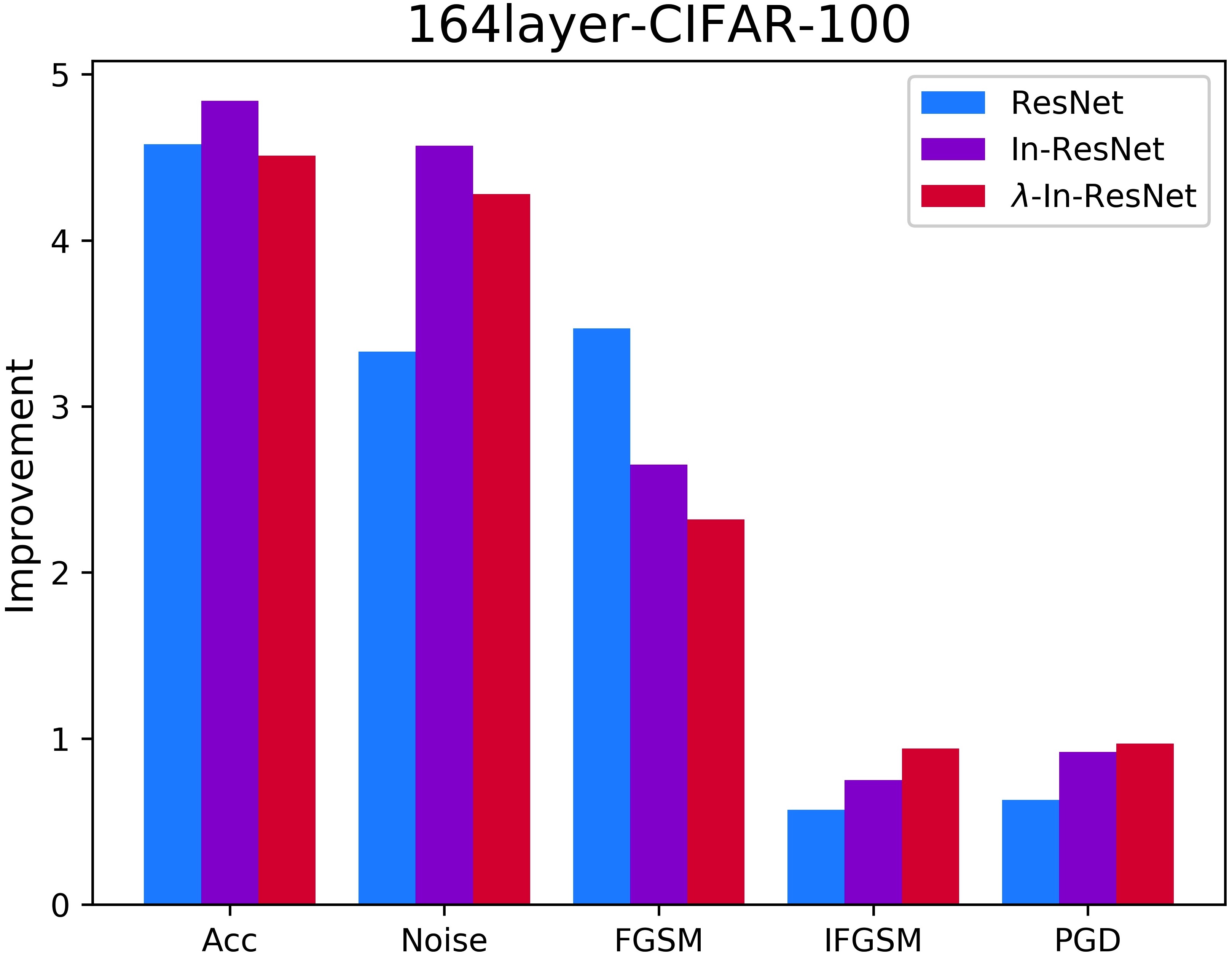}
}

\caption{The accuracy \textbf{improvements} over single models for (a) the ensemble ResNet-110, In-ResNet-110 and $\lambda$-In-ResNet-110; (b) the ensemble ResNet-164, In-ResNet-164 and $\lambda$-In-ResNet-164 over CIFAR-100 dataset. Both of the ensemble of our models have more significant accuracy improvements than the ensemble of the corresponding baseline model.}
\label{imp-100}
\end{figure*}

\newpage

\section{Standard Deviation for Reported Results}\label{appendix_d}

In Table \ref{stochastic-noise}, we reported model accuracy over stochastic noise from CIFAR-10-C and CIFAR-100-C datasets. In Table \ref{acc}, we reported model accuracy over unperturbed CIFAR-10 and CIFAR-100 test sets. In Table \ref{stochastic-adv}, we reported model accuracy over perturbed CIFAR-10 and CIFAR-100 images from FGSM, IFGSM, and PGD adversarial attacks with different attack radii. In Table \ref{init-results} and Table \ref{init-results-100}, we discussed about accuracy and robustness results of In-ResNet-110 and $\lambda$-In-ResNet-110 with different initialization schemes on CIFAR-10 and CIFAR-100.

Here we provide all of the standard deviation of the results about the performance of ResNet, In-ResNet and $\lambda$-In-ResNet. For simplicity, we assume that accuracy results over different stochastic noise groups in CIFAR-C are independent from each other. In general, the standard deviation scores of our models are comparable to those of the baseline ResNet models.

It should be noted that the following standard deviation results should be used together with the original results. While some method has low standard deviation, its averaged performance can be inferior as well. Also noted that some of the standard deviation scores of our models are larger than the baseline models. The large standard deviation scores are attributed to optimization difficulty, and result in more significant difference among different runs of our models. The performance difference also accounts for the fact that the ensemble method is more beneficial to our models. 

\

\begin{table*}[htbp]
    \centering
    \begin{tabular}{c|l||c|c|c|c|c}
    \hline
        Benchmark & Model & Impulse & Speckle & Gaussian & Shot & Avg. \\
    \hline
    \hline
        \multirow{6}{*}{CIFAR-10} & ResNet-110 & 1.451 & 2.356 & 3.003 & 2.570 & 2.412 \\
        ~ & In-ResNet-110 & 2.576 & 3.134 & 4.769 & 3.485 & 3.583 \\
        ~ & $\lambda$-In-ResNet-110 & 2.433 & 2.927 & 4.383 & 3.109 & 3.293 \\
    \cline{2-7}
        ~ & ResNet-164 & 1.262 & 2.058 & 3.136 & 2.443 & 2.325 \\
        ~ & In-ResNet-164 & 1.352 & 3.373 & 5.570 & 3.903 & 3.856 \\
        ~ & $\lambda$-In-ResNet-164 & 1.804 & 1.992 & 3.229 & 2.356 & 2.408 \\
    \hline
        \multirow{6}{*}{CIFAR-100} & ResNet-110 & 1.870 & 1.076 & 0.910 & 1.073 & 1.288 \\
        ~ & In-ResNet-110 & 1.489 & 2.792 & 3.250 & 3.135 & 2.757 \\
        ~ & $\lambda$-In-ResNet-110 & 0.651 & 1.468 & 1.874 & 1.589 & 1.468 \\          
    \cline{2-7}
        ~ & ResNet-164 & 1.376 & 2.051 & 1.717 & 1.817 & 1.757  \\
        ~ & In-ResNet-164 & 0.729 & 1.849 & 1.704 & 1.906 & 1.619 \\
        ~ & $\lambda$-In-ResNet-164 & 0.736 & 1.269 & 1.682 & 1.446 & 1.330 \\          
    \hline
    \end{tabular}
    \caption{Standard deviation of each accuracy results for ResNet, In-ResNet and $\lambda$-In-ResNet reported in Table \ref{stochastic-noise}. For simplicity, results under each types of noise are assumed as independent from each other.}
    \label{stochastic-noise-std}
\end{table*}

\begin{table}[htbp]
    \centering
    \begin{tabular}{l|c|c}
    \hline
        Model & CIFAR-10 & CIFAR-100 \\
    \hline
    \hline
        ResNet-110 & 0.396 & 0.144 \\
        In-ResNet-110 & 0.831 & 0.402 \\
        $\lambda$-In-ResNet-110 & 0.433 & 0.510 \\
    \hline        
        ResNet-164 & 0.368 & 0.224 \\
        In-ResNet-164 & 0.635 & 0.507 \\
        $\lambda$-In-ResNet-164 & 0.598 & 0.279 \\        
    \hline      
    \end{tabular}
    \caption{Standard deviation of each accuracy results for ResNet, In-ResNet and $\lambda$-In-ResNet reported in Table \ref{acc}.}
    \label{acc-std}
\end{table}

\begin{table*}[htbp]
    \centering
    \begin{tabular}{l|l||r|r|r|r|r|r|r|r|r}
    \hline
        \multirow{2}{*}{Benchmark} & \multirow{2}{*}{Model} & \multicolumn{3}{c|}{FGSM} & \multicolumn{3}{|c|}{IFGSM} & \multicolumn{3}{|c}{PGD} \\
    \cline{3-11}
        ~ &  ~ & 2/255 & 4/255 & 8/255 & 2/255 & 4/255 & 8/255 & 2/255 & 4/255 & 8/255 \\
    \hline
    \hline
        \multirow{6}{*}{CIFAR-10} & ResNet-110 & 0.782 & 0.577 & 0.894 & 1.166 & 0.549 & 0.030 & 1.257 & 0.456 & 0.021 \\
        ~ & In-ResNet-110 & 1.841 & 1.905 & 1.349 & 3.886 & 6.536 & 2.667 & 3.953 & 6.613 & 2.737 \\
        ~ & $\lambda$-In-ResNet-110 & 0.942 & 1.304 & 1.215 & 2.087 & 3.875 & 1.330 & 2.118 & 3.934 & 1.330 \\ 
    \cline{2-11}       
        ~ & ResNet-164 & 1.024 & 1.128 & 0.819 & 2.917 & 1.770 & 0.051 & 2.983 & 1.696 & 0.043 \\
        ~ & In-ResNet-164 & 1.192 & 1.450 & 1.058 & 2.220 & 4.284 & 0.983 & 2.279 & 4.296 & 0.988 \\
        ~ & $\lambda$-In-ResNet-164 & 1.633 & 2.066 & 1.662 & 3.095 & 4.731 & 0.867 & 3.200 & 4.779 & 0.876 \\
    \hline      
        \multirow{6}{*}{CIFAR-100} & ResNet-110 & 0.724 & 0.436 & 0.292 & 0.982 & 0.270 & 0.056 & 0.930 & 0.274 & 0.088 \\
        ~ & In-ResNet-110 & 1.593 & 0.840 & 1.796 & 2.904 & 1.124 & 0.150 & 2.976 & 1.157 & 0.143 \\
        ~ & $\lambda$-In-ResNet-110 & 0.851 & 0.515 & 0.488 & 1.339 & 0.645 & 0.145 & 1.373 & 0.610 & 0.143 \\     
    \cline{2-11}       
        ~ & ResNet-164 & 0.671 & 0.372 & 0.958 & 1.479 & 0.505 & 0.081 & 1.523 & 0.524 & 0.072 \\
        ~ & In-ResNet-164 & 1.196 & 0.996 & 0.661 & 1.946 & 1.078 & 0.181 & 1.976 & 1.058 & 0.170 \\
        ~ & $\lambda$-In-ResNet-164 & 0.871 & 0.732 & 0.662 & 1.105 & 0.703 & 0.125 & 1.127 & 0.703 & 0.134 \\
    \hline              
    \end{tabular}
    \caption{Standard deviation of each accuracy results for ResNet, In-ResNet and $\lambda$-In-ResNet reported in Table \ref{stochastic-adv}.}
    \label{stochastic-adv-std}
\end{table*}

\begin{table*}[htbp]
    \centering
    \begin{tabular}{l|r|r|r|r|r}
        \hline
        Model & Acc. & noise & FGSM & IFGSM & PGD \\
        \hline
        \hline
        ResNet-110 & 0.396 & 2.412 & 0.577 & 0.549 & 0.456 \\
        \hline        
        In-ResNet-110 & 0.831 & 3.583 & 1.905 & 6.536 & 6.613 \\
        In-ResNet-sig-110 & 0.157 & 2.042 & 1.004 & 0.775 & 0.813 \\
        In-ResNet-gating-110 & 0.204 & 1.322 & 0.669 & 0.270 & 0.223 \\
        In-ResNet-gating-sig-110 & 2.101 & 7.860 & 6.346 & 13.397 & 13.491 \\
        \hline
    \end{tabular}
    \caption{Standard deviation of each accuracy results reported in Table \ref{comp-results}.}
    \label{comp-results-std}
\end{table*}

\newpage

\begin{table}[htbp]
    \centering
    \begin{tabular}{l|c|r|r|r|r|r}
        \hline
        Model & Initialization & Acc. & noise & FGSM & IFGSM & PGD \\
        \hline
        \hline
        ResNet & - & 0.396 & 2.412 & 0.577 & 0.549 & 0.456 \\
        \hline
         & $\mathcal{U}[0.00, 0.10]$ & 0.208 & 2.375 & 0.336 & 0.999 & 0.911 \\
         & $\mathcal{U}[0.10, 0.20]$ & 0.224 & 3.650 & 1.329 & 4.742 & 4.820 \\
        In-ResNet &$\mathcal{U}[0.20, 0.25]$ & 0.831 & 3.583 & 1.905 & 6.536 & 6.613 \\        
         &$\mathcal{U}[0.25, 0.30]$ & 0.582 & 2.448 & 1.259 & 2.640 & 2.680 \\
        &$\mathcal{U}[0.30, 0.40]$  & 0.328 & 0.763 & 0.434 & 0.931 & 0.920 \\
        \hline           
         & $\mathcal{U}[0.00, 0.10]$ & 0.079 & 1.815 & 0.624 & 1.235 & 1.183 \\
         & $\mathcal{U}[0.10, 0.20]$ & 0.277 & 3.772 & 1.282 & 4.432 & 4.485 \\
        $\lambda$-In-ResNet & $\mathcal{U}[0.20, 0.25]$ & 0.433 & 3.293 & 1.304 & 3.875 & 3.934 \\ 
         &$\mathcal{U}[0.25, 0.30]$ & 0.617 & 3.091 & 1.319 & 3.674 & 3.734 \\
        &$\mathcal{U}[0.30, 0.40]$  & 0.233 & 0.219 & 0.544 & 0.997 & 1.004 \\
        \hline         
    \end{tabular}
    \caption{Standard deviation of each accuracy results reported in Table \ref{init-results}. Note that for $\mathcal{U}[0.3, 0.4]$, the reported S.D. results are calculated only over 4(2) successful runs.}
    \label{init-results-std}
\end{table}

\begin{table}[htbp]
    \centering
    \begin{tabular}{l|c|r|r|r|r|r}
        \hline
        Model & Initialization & Acc. & noise & FGSM & IFGSM & PGD \\
        \hline
        \hline
        ResNet & - & 0.144 & 1.288 & 0.436 & 0.270 & 0.274 \\
        \hline
         & $\mathcal{U}[0.00, 0.10]$ & 0.436 & 1.489 & 0.349 & 0.347 & 0.287 \\
         & $\mathcal{U}[0.10, 0.20]$ & 0.640 & 1.622 & 0.652 & 0.914 & 0.874 \\
        In-ResNet &$\mathcal{U}[0.20, 0.25]$ & 0.402 & 2.757 & 0.840 & 1.124 & 1.157 \\        
         &$\mathcal{U}[0.25, 0.30]$ & 0.836 & 2.053 & 1.174 & 1.348 & 1.399 \\
        &$\mathcal{U}[0.30, 0.40]$  & 1.342 & 2.725 & 1.237 & 2.025 & 1.921 \\
        \hline           
         & $\mathcal{U}[0.00, 0.10]$ & 0.308 & 1.565 & 0.341 & 0.603 & 0.579 \\
         & $\mathcal{U}[0.10, 0.20]$ & 0.573 & 2.228 & 0.422 & 0.894 & 0.903 \\
        $\lambda$-In-ResNet & $\mathcal{U}[0.20, 0.25]$ & 0.510 & 1.468 & 0.515 & 0.645 & 0.610 \\ 
         & $\mathcal{U}[0.25, 0.30]$ & 0.555 & 2.082 & 1.413 & 1.069 & 1.056 \\
        & $\mathcal{U}[0.30, 0.40]$  & 0.368 & 1.767 & 0.382 & 1.329 & 1.315 \\
        \hline         
    \end{tabular}
    \caption{Standard deviation of each accuracy results reported in Table \ref{init-results-100}. Note that for $\mathcal{U}[0.3, 0.4]$, the reported S.D. results are calculated only over 3(2) successful runs.}
    \label{init-results-100-std}
\end{table}

\end{document}